\documentclass{article}

\usepackage{PRIMEarxiv}
\usepackage{hyperref}       % hyperlinks
\usepackage{url}            % simple URL typesetting
\usepackage{xurl}
\usepackage{booktabs}       % professional-quality tables
\usepackage{amsfonts}       % blackboard math symbols
\usepackage{nicefrac}       % compact symbols for 1/2, etc.
\usepackage{lipsum}
\usepackage{graphicx}
\usepackage{caption} % To use caption with figure and images
\usepackage{subcaption}
\usepackage{adjustbox}
\usepackage{tabularx}
\usepackage{makecell}
\usepackage{nicefrac}       % compact symbols for 1/2, etc.
\usepackage{algorithm}
\usepackage{algpseudocode}
\usepackage{lscape}
\usepackage{mathtools}
\usepackage{array} % The array environment is used to make a table of information, with column alignment (left, center, or right) and optional vertical lines separating the columns
\usepackage{tcolorbox}
\usepackage{thmtools}
\usepackage[utf8]{inputenc} % allow utf-8 input
\usepackage{csquotes}
\ifdefined\theorem
\else
    \newtheorem{theorem}{Theorem}
\fi

\newenvironment{proof-sketch}{\par\noindent{\it Proof sketch.\ }}{\hfill\BlackBox\\[2mm]}

\usepackage{soul}
\renewcommand{\hl}[1]{#1}
\usepackage{cleveref}
\usepackage{multirow}
\usepackage[T1]{fontenc}    % use 8-bit T1 fonts
\usepackage{hyperref}       % hyperlinks
\usepackage{url}            % simple URL typesetting
\usepackage{booktabs}       % professional-quality tables
\usepackage{amsfonts}       % blackboard math symbols
\usepackage{nicefrac}       % compact symbols for 1/2, etc.
\usepackage{microtype}      % microtypography
\usepackage{lipsum}
\usepackage{fancyhdr}       % header
\usepackage{graphicx}       % graphics
\usepackage{natbib}
\graphicspath{{media/}}     % organize your images and other figures under media/ folder
\newcommand{\BlackBox}{\rule{1.5ex}{1.5ex}}  % end of proof
\ifdefined\proof
    \renewenvironment{proof}{\par\noindent{\bf Proof\ }}{\hfill\BlackBox\\[2mm]}
\else
    \newenvironment{proof}{\par\noindent{\bf Proof\ }}{\hfill\BlackBox\\[2mm]}
\fi

\title{Supersampling Stable Diffusion and Beyond: Training-Free Neural Network Scaling via Signal Interpolation}

\author{
  Md. Abu Obaida Zishan \\
  School of Computational Data and Sciences \\
  BRAC University \\
  Dhaka, Bangladesh\\
  \texttt{md.abu.obaida.zishan@g.bracu.ac.bd} \\
   \And
  Jannatun Noor \\
  Computing for Sustainability and Social Good research group\\
  United International University \\
  Dhaka, Bangladesh\\
  \texttt{jannatun@cse.uiu.ac.bd} \\
  \AND
  Annajiat Alim Rasel \\
  School of Computational Data and Sciences \\
  BRAC University \\
  Dhaka, Bangladesh\\
  \texttt{annajiat@gmail.com} \\
}

\begin{document}
\maketitle

\begin{abstract}
This work presents a training-free approach for interpolating neural networks using signal interpolation methods (i.e., linear, bilinear, bicubic). We discovered this phenomenon while attempting to interpolate the convolution layers of Stable-Diffusion-1.5 for beyond-training-resolution (BTR) image generation. As we know, for generating BTR images using DDPMs (Denoising Diffusion Probabilistic Models), there are mainly two approaches, namely, 1) patch-wise diffusion and 2) convolution-kernel dilation. Patch-wise diffusion cannot solve the object-duplication problem efficiently, while the kernel dilation approach may be harder to fine-tune later. Thus, we proposed convolution-kernel interpolation of the Stable Diffusion 1.5 model for $4\times$ BTR image generation. We discover and show mathematically that signal interpolation methods can correctly scale convolution kernels when multiplied by a flat-attenuation factor, while achieving competitive performance against the state of the art. We then applied the same methodology to non-DDPM deep neural network architectures, adapting them for BTR image classification, achieving a worst-case performance drop of $\sim 7\%$ on the ImageNet-1k validation set. Again, we subsample the penultimate fully connected layer of VGG-16 using our method, reducing up to $\sim 9\%$ parameters with $\sim 5\%$ loss in performance. Finally, we also discuss how we can reduce the memory footprint of training neural networks to $4\times$ or greater.
\end{abstract}

\section{Introduction}
% \linenumbers        % Start numbering
\begin{figure}[!ht]
    \centering
    \includegraphics[width=0.5\linewidth]{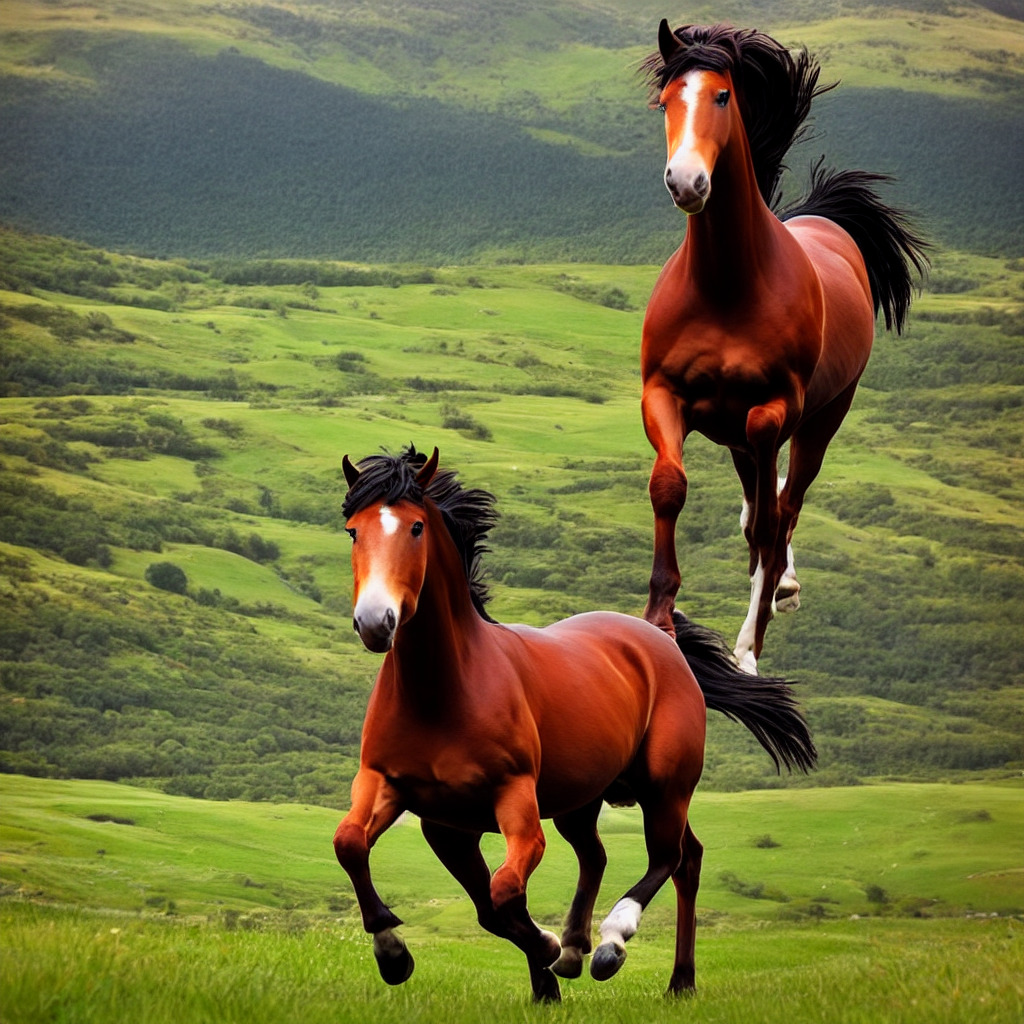}
    \caption{BTR generation using Stable Diffusion-1.5. A $1024\times1024$ image (model’s base/training: $512\times512$) produced with the prompt --- \textit{Generate a photorealistic picture of a horse, running on a beautiful valley}. This reveals the object duplication artifact, where two horses are produced instead of one.}
    \label{fig:obj-duplication}
\end{figure}

Recently, text-to-image generative models improved significantly on generating near-realistic images while reducing the effects of hallucination with support for arbitrary resolutions \citep{nanobanana, openai2025gpt4oimage, wu2025qwenimagetechnicalreport, labs2025flux1kontextflowmatching}. However, these models often require billions of parameters, making them infeasible for consumers to generate images from their personal workstations. Again, the dramatic inflation of RAM prices has made buying subscriptions from software-as-a-service (SaaS) or infrastructure-as-a-service (IaaS) providers near-obligatory to run or deploy commercial and open-source generative models, respectively \citep{micronexitram, bbc2026articleram}. Additionally, the cost of graphics processing units (GPUs) has been inflating for some time already, which was further boosted by the RAM price inflation \citep{gpuramprice}. Thus, the research community was influenced to develop beyond-training-resolution (BTR) images (for a range of tasks such as text-to-image, image-to-image, etc.) without training or fine-tuning, by modifying the Stable-Diffusion family of models \citep{ldm-main, he2023scalecrafter, fouriscale2024, du2024demofusion, zhang2025hidiffusion, lin2024accdiffusion}.

\subsection{Problem Statement}
The Stable Diffusion family of models is unable to generate BTR images, as they hallucinate and duplicate desired objects from the prompt, as shown in \Cref{fig:obj-duplication}. To solve this issue through a training-free method, two approaches have become popular, namely, (1) patch-wise diffusion and (2) convolution-kernel dilation. Through the patch-wise diffusion method, the key complication was object-duplication, where even though BTR images were generated, they often consisted of object-duplication artifacts \citep{du2024demofusion, lin2024accdiffusion}. Again, the kernel dilation approach may be harder to fine-tune as they are less similar to the ideal supersampled kernel. 

\subsection{Contributions}
Thus, in this work, we attempt to solve the drawbacks of previous approaches. To this end, we have the following contributions:
\begin{itemize}
    \item We provide a theoretical basis for dilation being a special case of interpolation.
    \item We generate BTR images by super-sampling the convolution kernels (\textit{bicubic}) of SD-1.5.
    \item We attain competitive results in terms of $4\times$ image generation against the state of the art (SOTA).
    \item  We scale (training-free) image classification models for $4\times$ input through kernel-interpolation.
    \item Finally, we shrink $\sim 9\%$ parameters of VGG-16 with $\sim 5\%$ loss in performance using our approach.
\end{itemize}

In the next section, we provide a brief outline of this work.

\subsection{Organization}
In \Cref{sec:litreview}, we review related works while discussing their drawbacks. In \Cref{sec:methodology}, we discuss our methodology in detail. Then, in \Cref{sec:experiments}, we evaluate our approach empirically against the SOTA and discuss our findings in \Cref{sec:discussion}. Finally, in \Cref{sec:conclusion}, we discuss our limitations and provide suggestions for future work.

In the next section, we discuss related literature and its drawbacks.

\section{Literature Review}
\label{sec:litreview}
In this section, we discuss previous works on training-free methods for BTR image generation using pretrained diffusion models. We observe related works to be broadly divided into two approaches, namely, (1) patch-wise diffusion and (2) dilated convolution kernels. In patch-wise diffusion, images are generated in patches, where each patch is generated within the training resolution. On the other hand, dilated convolution involves dilating the convolution kernels (most of them) of the diffusion models to generate BTR images. 

In \citep{du2024demofusion}, the authors developed \textit{DemoFusion} by building on the patch-wise denoising method of MultiDiffusion \citep{bar2023multidiffusion}. The authors introduce three novel procedures, namely, progressive-upscaling, skip-residual, and dilated sampling, in addition to generating the image in overlapping patches. This resulted in the generation of BTR images with more coherent structure compared to previous works. However, the problem with repeating objects persisted. To address this problem, \cite{lin2024accdiffusion} introduced patch-content aware text-prompts and dilated sampling with window-interactions, solving the repetition issue. The patch-content aware method works by conditioning each patch through a sub-prompt split from the main prompt, only including prompt-tokens to be generated in that patch.

However, the major drawback of patch-wise diffusion is excess compute and memory utilization, as the image is generated in \textit{overlapping} patches. This issue is resolved by the dilated-convolution approach. \textit{ScaleCrafter} \citep{he2023scalecrafter} provides a way for generating BTR images through re-dilated and dispersed convolution for SD models. Additionally, it utilizes noise-damped classifier-free guidance to improve the denoising capability of the model.  Finally, \cite{fouriscale2024} (\textit{FouriScale}) observed a distribution gap in the frequency domain between high- and low-resolution images due to spatial down-sampling after using dilated convolutions. Hence, they applied low-pass filtering before convolution operations, consequently improving the image quality. Additionally, the authors modified the context-free-guidance (CFG) by generating an extra noise estimate with a milder low-pass filter to replace the self-attention scores of the strictly filtered noise estimate, improving image quality.

Even though dilating convolution kernels solved the core issue of object duplication, the solution may not be fine-tuning friendly. Thus, in the next section, we discuss our methodology of replacing dilation with interpolated convolution kernels.

\section{Proposed Methodology}
\label{sec:methodology}

\begin{figure}
\centering
    \begin{subfigure}{0.5\textwidth}
        \includegraphics[width=\linewidth]{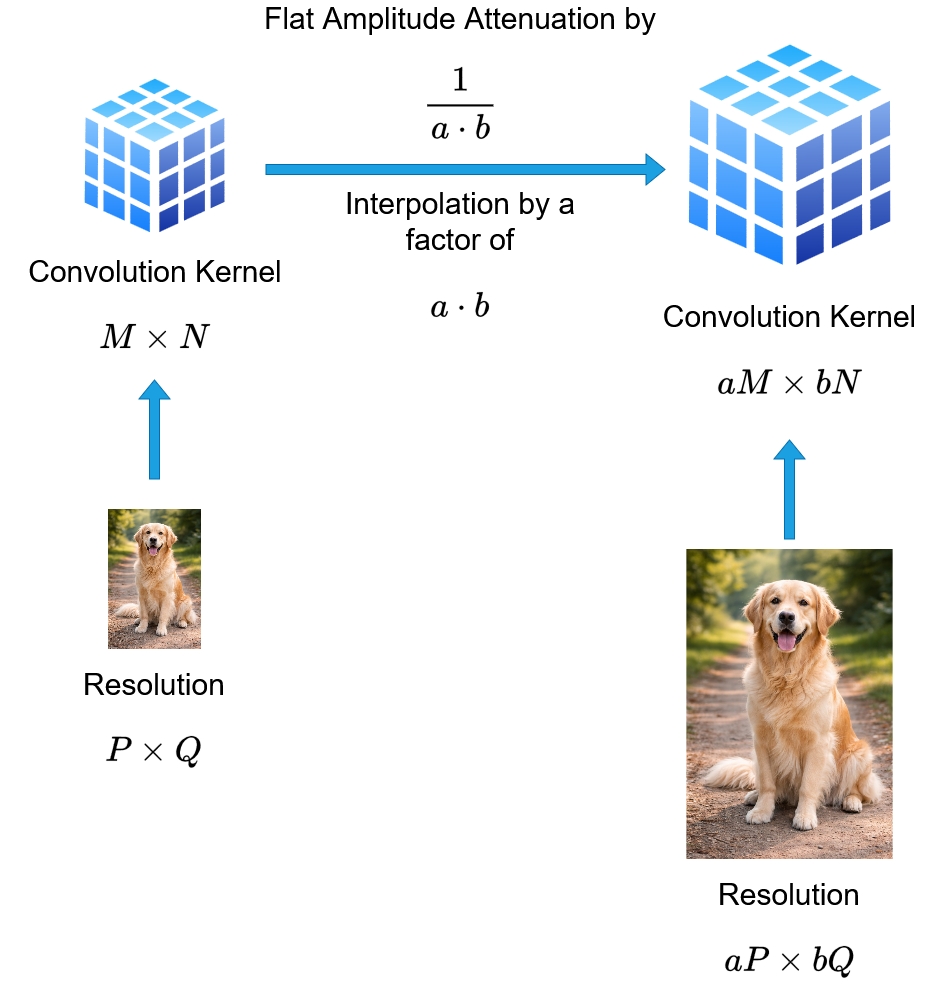}
        \caption{BTR image generation in Stable Diffusion models by convolution-kernel interpolation}
        \label{fig:framework-sd}
    \end{subfigure}
    \begin{subfigure}{0.7\textwidth}
        \centering
        \includegraphics[width=\linewidth]{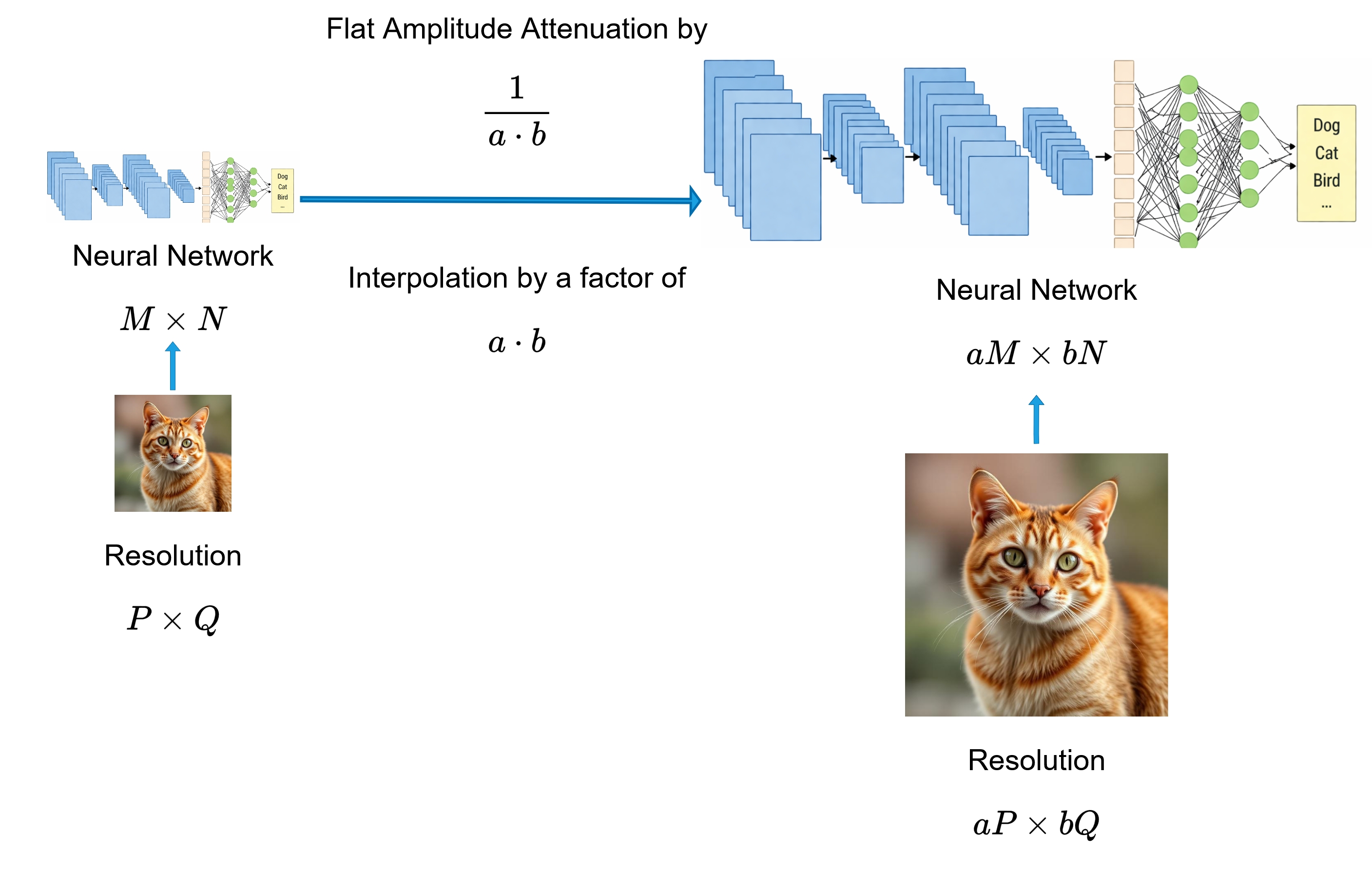}
        \caption{Supersampling of deep neural networks for accepting BTR inputs}
        \label{fig:framework-nn}
    \end{subfigure}
    \caption{\hl{An overview of the interpolation and attenuation framework presented in this work.}}
    \label{fig:framework}
\end{figure}

In this work, we propose the interpolation of fully connected and convolution layers instead of dilation \citep{smootheddilationwang}. Our approach is backed by both theoretical and empirical analysis, where we use \textit{bicubic} and \textit{bilinear} methods for interpolation \citep{zhu2022efficientbicubic}. To the theoretical end, we first show that vanilla interpolation of convolution kernels and fully connected layers do not work, as the underlying dot product is amplified by the factor of interpolation (\Cref{th:dot-main}).

\begin{restatable}{theorem}{ratiotheorem}
\label{th:dot-main}
The dot product of two vectors of odd length is equal to the dot products of its supersampled counterparts scaled by the inverse of the supersampling factor
\end{restatable}

\begin{proof-sketch}
Let, $\mathbf{D},\mathbf{E}$ be a pair of $N$-dimensional vectors. We can express each of $\mathbf{D},\mathbf{E}$ vectors as a sum of discrete cosine waves with $\lfloor\frac{N}{2}\rfloor+1$ unique frequencies $k\in [0, \lfloor\frac{N}{2}\rfloor]$: 

\[
  \mathbf{D} = \left(\sum_{k=0}^{\lfloor\frac{N}{2}\rfloor} A_{k} \cos\left(2\pi\frac{k}{N}n + \phi_{k}\right)\right)_{n=0}^{N-1}  
\]

\[
  \mathbf{E} =  \left(\sum_{k=0}^{\lfloor\frac{N}{2}\rfloor} B_{k} \cos\left(2\pi\frac{k}{N}n + \theta_{k}\right)\right)_{n=0}^{N-1}
\]

Let $\operatorname{I}_s(\mathbf{V}, f)$ be an ideal supersampling function that supersamples a vector $\mathbf{V}$ given a factor $f$. Thus we supersample $\mathbf{D},\mathbf{E}$ by a factor $\alpha$ where $\alpha > 1$.

\[
    \mathbf{F} = \operatorname{I}_s(\mathbf{D}, \alpha) = \left(\sum_{k=0}^{\lfloor\frac{N}{2}\rfloor} A_{k} \cos\left(2\pi\frac{k}{\alpha N}n + \phi_{k}\right)\right)_{n=0}^{\alpha N-1}  
\]

\[
    \mathbf{G} = \operatorname{I}_s(\mathbf{E}, \alpha) =\left(\sum_{k=0}^{\lfloor\frac{N}{2}\rfloor} B_{k} \cos\left(2\pi\frac{k}{\alpha N}n + \theta_{k}\right)\right)_{n=0}^{\alpha N-1}  
\]

For $\mathbf{D}, \mathbf{E}, \mathbf{F}, \mathbf{G}$ we assume $\phi_0,\theta_0 = 0$ for all $N$ and $\phi_{\lfloor\frac{N}{2}\rfloor},\theta_{\lfloor\frac{N}{2}\rfloor} \in \left\{ 0, \pi \right\}$ when $N$ is even. The dot product between $\mathbf{D},\mathbf{E}$ can be shown using Parseval's theorem.

\[
   \langle \mathbf{D},\mathbf{E} \rangle= \frac{1}{N} \left( \sum_{k=0}^{N-1} \widehat{D}[k] \cdot \overline{\widehat{E}[k]} \right)
\]

\begin{equation}
\label{eq:dot-prod-sk}
   \langle \mathbf{D},\mathbf{E} \rangle= N \left( A_0 B_0 + \sum_{k=1}^{\lfloor\frac{N-1}{2}\rfloor} \frac{A_k B_k}{2} \cos\left(\phi_k-\theta_k\right) + \operatorname{\delta_{even}}(N) A_{\frac{N}{2}} B_{\frac{N}{2}} \cos\left(\phi_{\frac{N}{2}}-\theta_{\frac{N}{2}}\right) \right)
\end{equation}

Similarly for $\mathbf{F},\mathbf{G}$ we can show:
\[
   \langle \mathbf{F},\mathbf{G} \rangle= \frac{1}{\alpha N} \left( \sum_{k=0}^{\alpha N-1} \widehat{F}[k] \cdot \overline{\widehat{G}[k]} \right)
\]

\[
    \langle \mathbf{F},\mathbf{G} \rangle= \alpha N \left( A_0 B_0 + \sum_{k=1}^{\lfloor\frac{\alpha N-1}{2}\rfloor} \frac{A_k B_k}{2} \cos\left(\phi_k-\theta_k\right) + \operatorname{\delta_{even}}(\alpha N) A_{\frac{\alpha N}{2}} B_{\frac{\alpha N}{2}} \cos\left(\phi_{\frac{\alpha N}{2}}-\theta_{\frac{\alpha N}{2}}\right) \right)
\]

However, as $\mathbf{F},\mathbf{G}$ are supersampled versions of $\mathbf{D},\mathbf{E}$ respectively, $\alpha \widehat{D}[k] = \widehat{F}[k]$ and $\alpha \widehat{E}[k] = \widehat{G}[k]$ for all $k \in [-\lfloor \frac{N-1}{2}\rfloor, \lfloor \frac{N-1}{2}\rfloor]$. For even $N$, we also have $\frac{\alpha}{2}\widehat{D}[\frac{N}{2}] = \widehat{F}[\pm \frac{N}{2}]$ and $\frac{\alpha}{2}\widehat{E}[\frac{N}{2}] = \widehat{G}[\pm \frac{N}{2}]$. For the rest of the frequencies, $|k| > \lfloor \frac{N}{2}\rfloor$, $\widehat{F}[k]=\widehat{G}[k]=0$.

\begin{equation}
\label{eq:dot-prod-an-sk}
     \langle \mathbf{F},\mathbf{G} \rangle= \alpha N \left( A_0 B_0 + \sum_{k=1}^{\lfloor\frac{N-1}{2}\rfloor} \frac{A_{k} B_{k}}{2} \cos\left(\phi_{k}-\theta_{k}\right) + \operatorname{\delta_{even}}(N) \frac{A_{\frac{N}{2}} B_{\frac{N}{2}}}{2} \cos\left(\phi_{\frac{N}{2}}-\theta_{\frac{N}{2}}\right) \right)  
\end{equation}

 For odd $N$ or $A_{\frac{N}{2}}B_{\frac{N}{2}}=0$, the Nyquist term vanishes, yielding the identity:

\begin{equation}
\label{eq:main-true}
     \frac{1}{\alpha} \langle \mathbf{F},\mathbf{G} \rangle= N \left( A_0 B_0 + \sum_{k=1}^{\lfloor\frac{N-1}{2}\rfloor} \frac{A_{k} B_{k}}{2} \cos\left(\phi_{k}-\theta_{k}\right) \right)  
\end{equation}

\begin{equation}
    \label{eq:equality}
     \frac{1}{\alpha} \langle \mathbf{F},\mathbf{G} \rangle = \langle \mathbf{D},\mathbf{E} \rangle
\end{equation}
\end{proof-sketch}

Therefore, to maintain equality of a supersampled dot product with its base counterpart, we need to scale it by the inverse of its supersampling factor for vectors with odd lengths. For even lengths, we get an error, $\kappa =  N \frac{A_{\frac{N}{2}} B_{\frac{N}{2}}}{2} \cos\left(\phi_{\frac{N}{2}}-\theta_{\frac{N}{2}}\right)$:
\[
 \langle \mathbf{F},\mathbf{G} \rangle= \alpha N \left( A_0 B_0 + \sum_{k=1}^{\lfloor\frac{N-1}{2}\rfloor} \frac{A_{k} B_{k}}{2} \cos\left(\phi_{k}-\theta_{k}\right) +\frac{A_{\frac{N}{2}} B_{\frac{N}{2}}}{2} \cos\left(\phi_{\frac{N}{2}}-\theta_{\frac{N}{2}}\right) \right) 
\]

\[
 \frac{1}{\alpha} \langle \mathbf{F},\mathbf{G} \rangle = \langle \mathbf{D},\mathbf{E} \rangle - N \frac{A_{\frac{N}{2}} B_{\frac{N}{2}}}{2} \cos\left(\phi_{\frac{N}{2}}-\theta_{\frac{N}{2}}\right)
\]

\[
 \frac{1}{\alpha} \langle \mathbf{F},\mathbf{G} \rangle = \langle \mathbf{D},\mathbf{E} \rangle -  \kappa
\]

For sufficiently small $|\kappa|$, we can generalize \Cref{eq:main-true} for both even and odd lengths as an approximation.
\begin{equation}
\label{eq:main-approx}
\boxed{
     \frac{1}{\alpha} \langle \mathbf{F},\mathbf{G} \rangle \approx N \left( A_0 B_0 + \sum_{k=1}^{\lfloor\frac{N-1}{2}\rfloor} \frac{A_{k} B_{k}}{2} \cos\left(\phi_{k}-\theta_{k}\right) \right)
}
\end{equation}

Due to the gain in the supersampled dot-product, the activation consistency of a neuron (which applies to convolution kernels as well) is no longer preserved as the biased remains unchanged. Dilation solves this problem by supersampling through zero-gapping the weight kernel by the factor of dilation. This scales the supersampled dot product such that $\langle \mathbf{F}_d, \mathbf{G_d} \rangle = \langle \mathbf{D}, \mathbf{E} \rangle$ instead of $\langle \mathbf{F}, \mathbf{G}\rangle = \alpha \left(\langle \mathbf{D}, \mathbf{E} \rangle+\kappa\right)$ as shown in \Cref{eq:main-true}. Here, $\mathbf{F}_d, \mathbf{G_d}$ are the dilated vectors and $\alpha$ is the supersampling or dilating factor. This phenomenon is shown in \Cref{th:dilation-dot}.

\begin{restatable}{corollary}{dilationcorollary}
\label{th:dilation-dot}
Dilation is a specialized interpolation that supersamples two vectors rendering their dot product to be equal to that of their base counterpart.
\end{restatable}
\begin{proof}
We take $\operatorname{I}_{d}$ as the dilating supersampling function. Given, dilation factor $\alpha > 1$, the function dilates previously defined $\mathbf{D}, \mathbf{E}$ to $\mathbf{F}_d, \mathbf{G}_d$.

\[
   \mathbf{F}_d = \operatorname{I}_d(\mathbf{D}, \alpha) = \left(\delta_{n \equiv 0 \bmod \alpha} \sum_{k=0}^{\lfloor\frac{N}{2}\rfloor}  A_{k} \cos\left(2\pi\frac{k}{\alpha N}n + \phi_{k}\right)\right)_{n=0}^{\alpha N-1}  
\]
 \[
   \mathbf{G}_d = \operatorname{I}_d(\mathbf{E}, \alpha) = \left(\delta_{n \equiv 0 \bmod \alpha} \sum_{k=0}^{\lfloor\frac{N}{2}\rfloor}  B_{k} \cos\left(2\pi\frac{k}{\alpha N}n + \theta_{k}\right)\right)_{n=0}^{\alpha N-1}  
\] 
The dot product between $\mathbf{F}_d, \mathbf{G}_d$ can be shown as:
\[
   \langle \mathbf{F}_d,\mathbf{G}_d \rangle= \sum_{n=0}^{\alpha N-1}F_d[n] \cdot G_d[n]
\]
By definition, $F_d[\alpha n] = D[n]$ and $G_d[\alpha n] = E[n]$ for $n \in [0, N-1]$. Rest of the values in $\mathbf{F}_d, \mathbf{G}_d$ are zero. 
\[
   \langle \mathbf{F}_d,\mathbf{G}_d \rangle= \sum_{n=0}^{N-1}F_d[\alpha n] \cdot G_d[\alpha n]
\]
\[
   \langle \mathbf{F}_d,\mathbf{G}_d \rangle= \sum_{n=0}^{N-1}D[n] \cdot E[n]
\]
\[
   \langle \mathbf{F}_d,\mathbf{G}_d \rangle= \langle \mathbf{D},\mathbf{E} \rangle
\]
\end{proof}

Thus, in this work, we attenuate and interpolate fully connected and convolution layers following \Cref{eq:main-approx}. The complete proof of \Cref{th:dot-main} and implementation details are provided in \Cref{sec:theorems-and-proofs,sec:impl-sd-1.5}, respectively. In the next section, we discuss our experimental approach and findings to validate our methodology.

\section{Experimental Evaluation}
\label{sec:experiments}
This section discusses our experimental setup, our findings, and comparison with related works. To do so, we download 3K images from the LAION-5B dataset to compare Frechet Inception Distance (FID) \citep{fid} and Kernel Inception Distance (KID) \citep{kid} against the generated images.  Again, using the same principle of kernel interpolation, we also interpolate deep neural networks, adapt them for BTR inputs, and compare their performance against a no-training baseline. However, for this task, we use the ImageNet-1k validation dataset \citep{imagenet}.

\subsection{Experimental Setup}

We perform our experiments in a workstation consisting of an Intel\textsuperscript{®} Xeon\textsuperscript{®} processor \citep{intelxeon}, variable (but plenty of) RAM, and an NVIDIA A100 GPU \citep{nvidiaa100} in our remote virtual machine (VM). We use both 40 and 80 GB variants of the GPU (one at a time), depending on their availability from our platform-as-a-service provider.  

\subsection{Experimental Metrics}
To conduct our experiments in \Cref{sec:btr} , we use Fr\'{e}chet Inception Distance (FID) \citep{fid} and Kernel Inception Distance (KID) metrics \citep{kid}. However, we use variants of the metrics namely, FID\textsubscript{b}, FID\textsubscript{r}, KID\textsubscript{b}, and KID\textsubscript{r} as done by \citep{fouriscale2024, he2023scalecrafter}. Again, for \Cref{sec:dnn-interp} we use accuracy and F1-Score.

\subsection{BTR Image Generation}
\label{sec:btr}
For empirical evaluation of our methodology, we generate 3K images using captions from the corresponding real images in the LAION-5B dataset. To ensure a fair experiment, we use the same captions to generate images for the related work. Then we utilize the selected experimental metrics for benchmarking. To generate images for this experiment, we used \texttt{sd-legacy/stable-diffusion-v1-5} from Hugging Face\footnote {The evaluated model can be found at \url{https://huggingface.co/stable-diffusion-v1-5/stable-diffusion-v1-5}}. For comparison with related works, we generate images $4\times$ of the base resolution ($512\times512$) of SD-1.5, that is $1024\times1024$ as shown in \Cref{tab:results-sota}. 

\begin{figure}
    \centering
    \includegraphics[width=\linewidth]{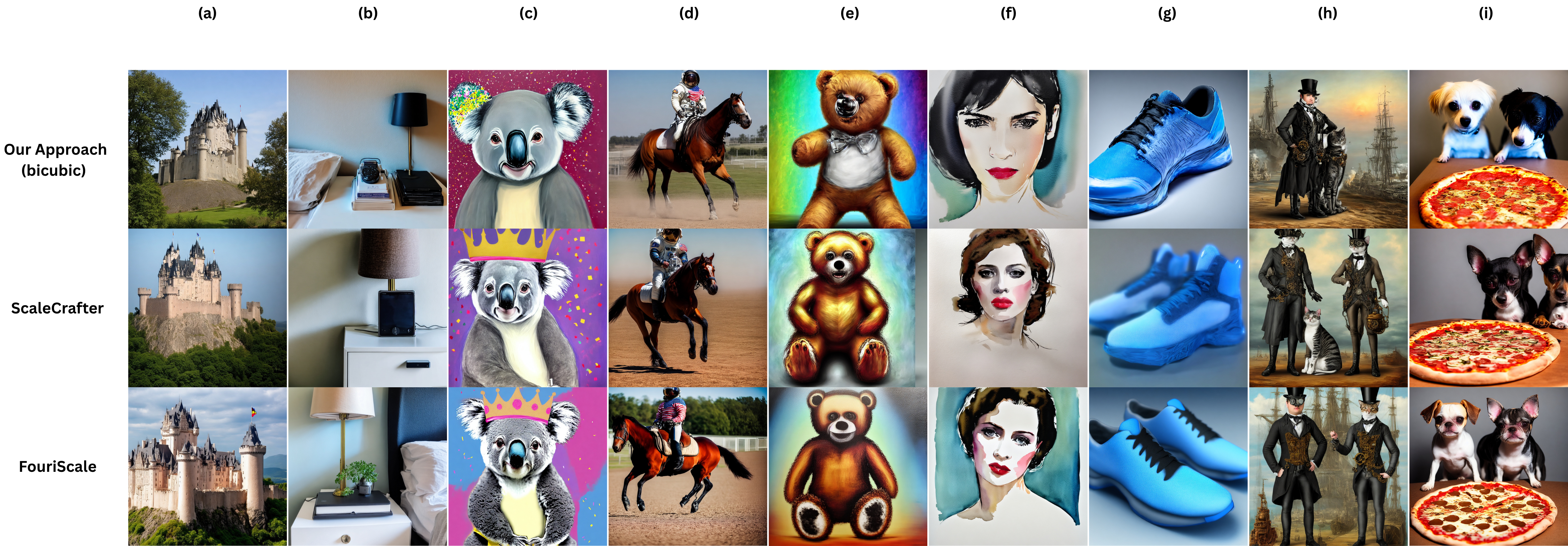}
    \caption{Picture quality comparison of our, ScaleCrafter, and FouriScale's approach. We can observe minimal difference in picture quality, as reflected in the analysis in \Cref{tab:results-sota}. The prompts of the pictures -- numbered (a) to (i) -- can be found in \Cref{tab:qual-prompts}. The pictures generated here are of $1024\times1024$ resolution, which is $4\times$ the base resolution of SD-1.5 $(512\times512)$.}
    \label{fig:photocollage}
\end{figure}  

\begin{table}[ht]
    \centering
    \begin{tabular}{ccccc}
        \toprule
        Method& FID\textsubscript{r}$\downarrow$ & KID\textsubscript{r}$\downarrow$ & FID\textsubscript{b}$\downarrow$ & KID\textsubscript{b}$\downarrow$\\
        \midrule
        FouriScale \citep{fouriscale2024} & \textbf{7.47} & \textbf{0.0035} & \textbf{11.82} & \textbf{0.0026} \\
        ScaleCrafter \citep{he2023scalecrafter} & \underline{7.92} & \underline{0.0038} & \underline{12.23} & \underline{0.0028}\\     
        Ours (bicubic) & 8.67 & 0.004 & 14.59 & 0.0034 \\
        Ours (bilinear) & 9.42 & 0.0045 & 16.57 & 0.0045\\        
        \bottomrule
    \end{tabular}
    \caption{Results for calculating FID and KID metrics using the SD-1.5 model for $4\times$ image generation $(1024\times1024)$. Best and second-best results are in bold and underlined, respectively.}
    \label{tab:results-sota}
\end{table}

\begin{table}[ht]
    \centering
    \begin{tabularx}{\textwidth}{c|X} 
        \toprule
        No. & Prompt \\
        \midrule
        (a) & A castle is in the middle of a european city \\
        (b) & A nighstand topped with a white land-line phone, remote control, a metallic lamp, and a black hardcover book. \\
        (c) & A painting of a koala wearing a princess dress and crown, with a confetti background. \\
        (d) & A professional photograph of an astronaut riding a horse \\ 
        (e) & A teddy bear mad scientist mixing chemicals depicted in oil painting style \\
        (f) & A watercolor portrait of a woman by Luke Rueda Studios and David Downton. \\
        (g) & Side-view blue-ice sneaker inspired by Spiderman created by Weta FX \\
        (h) & Two cats, grey and black, are wearing steampunk attire and standing in front of a ship in a heavily detailed painting. \\
        (i) & Two little dogs looking a large pizza sitting on a table \\
        \bottomrule
    \end{tabularx}
    \caption{Prompts used in picture quality comparison in \Cref{fig:photocollage} and are provided verbatim. Hence, some prompts missing a period (.) are intentional among other things.}
    \label{tab:qual-prompts}
\end{table}

In \Cref{tab:results-sota}, we observe that our method achieves competitive performance against the state-of-the-art but does not exceed it. But an interpolated kernel may converge faster than a dilated kernel during fine-tuning (see \Cref{sec:interp-slow}). \Cref{fig:photocollage} shows a qualitative comparison, where we see a near-negligible difference in the quality of the generated pictures. 

In the next section, we interpolate three Deep Neural Network (DNN) architectures to adapt for higher-resolution inputs, demonstrating the generalizability of our approach.

\subsection{Deep Neural Network Upscaling}
\label{sec:dnn-interp}
In this section, we discuss the quantitative results of interpolating deep neural networks for accepting BTR input. To do so, we interpolate three neural network architectures, namely VGGNet \citep{simonyan2015deepconvolutionalnetworkslargescalevggnet}, ResNet \citep{he2016deepresnet}, and ViT \citep{vit}. All of them have been trained/designed to use $224\times224$ images; hence, we adapt them to accommodate $448\times448$ resolution, without significant loss of performance through training-free interpolation. For testing our method's applicability, we use the ImageNet-1k validation set \citep{imagenet} \footnote{The dataset was downloaded from \url{https://huggingface.co/datasets/ILSVRC/imagenet-1k}}.

For all architectures, we interpolate all convolution layers from $H\times W$ to $2*H-1\times2*W-1$ where $H,W$ are height and width, respectively. We also update the padding from $(\lfloor \frac{H}{2} \rfloor, \lfloor \frac{W}{2} \rfloor)$ to $(\lfloor \frac{2*H-1}{2} \rfloor, (\lfloor \frac{2*W-1}{2} \rfloor)$, ensuring that the output before the first fully-connected layer (and the last convolution layer) is $4\times$ of the base output.

Again, when it comes to interpolating fully connected layers, we only interpolate the first fully connected layer and ignore the rest of the consecutive layers, due to the lack of necessity for adapting to a higher resolution. Hence, for a fully-connected layer of weights with shape $(\mathtt{output\_features, input\_features})$ we reshape to --- 
\[(\mathtt{output\_features, C,\frac{\sqrt{input\_features}}{C}, \frac{\sqrt{input\_features}}{C}})\] 

Here, $C$ is the number of channels before flattening of the input for passing to the (first) fully-connected layer. Then, we interpolate the layer such that the output shape of the weights becomes --- 
\[
(\mathtt{output\_features, C,\frac{m*\sqrt{input\_features}}{C}, \frac{n*\sqrt{input\_features}}{C}})
\]
where ---
\[
    m = \frac{\text{Height of a supersampled input}}{\text{Height of a base input}}
\] 

\[
    n = \frac{\text{Width of a supersampled input}}{\text{Width of a base input}}
\]
Similar to interpolating convolution layers, we do not modify the bias. Again, we utilize bicubic interpolation for interpolating both convolution and fully-connected layers like before. In the latter sections, we discuss the model-specific adaptations made for interpolating the selected architectures.

\subsubsection{VGGNet}
To interpolate VGGNet, we chose VGGNet-16's D variant from \citep{simonyan2015deepconvolutionalnetworkslargescalevggnet}. \Cref{tab:vgg16bvi} shows how the architecture of the interpolated model has been adapted for $448\times448$ resolution. Along with the convolution layers, the first linear layer is also interpolated, showcasing the general use case of this work. We also provide the result of the same variant with an Adaptive Average Pooling (AAP) layer of $(7,7)$ (before the first fully connected layer). This allows us to reduce the parameters by avoiding the interpolation of the first fully connected layer, also allowing slightly better results. The results for the evaluation dataset can be observed in \Cref{tab:resultsallmodelsinterp}.

\subsubsection{ResNet-18}
Similar to VGG-16, we also interpolate ResNet-18 \citep{he2016deepresnet} following the procedure mentioned earlier. Therefore, we interpolate the convolution layers only, as an AAP layer exists before the fully connected layer. For brevity, we skip providing the architectural comparison, but it is similar to VGG-16, and in accordance with the earlier discussion. \Cref{tab:resultsallmodelsinterp} shows the performance comparison of the base and interpolated variants of the model.

\begin{table}[ht]
\centering
\caption{Performance comparison of the base and interpolated VGG-16, ResNet-18, and, ViT-B-16. Here, the base and interpolated scores refer to the performance on the ImageNet-1k validation set with sizes $(224,224)$ and $(448, 448)$, respectively. The interpolated versions were never trained, but were supersampled from the trained weights of the base models. We also include the results of interpolating without flat-attenuation and observe a significant performance drop on those cases. This validates our core theoretical contribution on successfully applying interpolation for training-free scaling of neural network models.}
\begin{tabular}{|c|c|c|c|c|c|}
\hline

\text{} & \textbf{Variant}  & \textbf{Accuracy} & \textbf{F1-score} \\
\hline
\multirow{3}{*}{\textbf{VGG-16}} & Base  & 0.716 & 0.712\\ 
&Interp.  & 0.141 & 0.155\\
&Interp. + FA & 0.6613 & 0.6573\\
&Interp. + FA + AAP& 0.6617 & 0.6577 \\
\hline
\multirow{2}{*}{\textbf{ResNet-18}} & Base  & 0.698 & 0.693\\
&Interp. &  0.001 & 0.000\\
&Interp. + FA & 0.629 & 0.627\\
\hline
\multirow{2}{*}{\textbf{ViT-B-16}} & Base  & 0.811 & 0.808\\
&Interp. & 0.334 & 0.363\\
&Interp. + FA & 0.809 &  0.807\\
\hline
\end{tabular}
\label{tab:resultsallmodelsinterp}
\end{table}

\subsubsection{ViT-B-16}
Similar to previous architectures, we also interpolate ViT-B-16 \citep{vit}. However, this architecture only has a single convolution layer for projection. We interpolate that layer while doubling the stride to accommodate a $448\times448$ resolution. However, we do not interpolate the fully connected layers, as their input does not change due to the increase in image size, due to self-attention in the previous layer. The performance of base vs interpolated ViT-B-16 can be observed in \Cref{tab:resultsallmodelsinterp}.

\Cref{tab:resultsallmodelsinterp} shows the necessity of flat-attenuation for the interpolation operation to preserve the performance of the models. We observe that after interpolation, performance remains close to that of the base variants (Interp. + FA / Interp. + FA + AAP). The slight difference in performance may be minimized with a few epochs of fine-tuning, substantially reducing training compute costs.  Furthermore, the results show that we do not necessarily need to interpolate the first fully-connected layer after convolutions. Rather, using an AAP layer provides slightly better performance with fewer parameters, as seen in the case of VGG-16. However, without flat-attenuation (Interp.), we observe the performance to be worse than a random guess in the best case scenario (ViT-B-16) and nearly zero in the worst case (ResNet-18).

\subsection{Deep Neural Network Downscaling}
\label{dnn-down}
To demonstrate the applicability of our method in downscaling neural networks, we downsample the fully connected layer of the VGG-16 (D) model marked * in \Cref{tab:vgg16bvi}. However, we do not reduce the number of neurons (output features) but the input features of that layer. To downsample, we override the forward method of the marked FC layer to shrink its input as per the rates provided in \Cref{tab:vgg16down}. Furthermore, we downsample each neuron of the fully-connected layer by the same rates following our methodology with flat attenuation. In both cases, we utilize linear interpolation.  In \Cref{tab:vgg16down}, we observe that, for up to 10\% downscaling of the layer, the accuracy and F1-score are higher than 50\% on the ImageNet validation set of $224\times224$ images. Furthermore, without flat attenuation, the performance degrades slightly, which is expected according to our methodology.

\begin{table}[h]
    \centering
    \begin{tabular}{ccccccc}\toprule
    Downscaling (\%) & Attn. Factor & FC layer shape & Input Shape & Accuracy & F1-Score & Params. (M)\\\midrule
         0  & 1   & $4096 \times 4096$ & 4096 & 0.716 & 0.712 & 138.358\\
         50 & 2   & $4096 \times 2048$ & 2048 & 0.686 & 0.679 & 129.969\\
         25 & 4   & $4096 \times 1024$ & 1024 & 0.661 & 0.657 & 125.775\\
         10 & 10  & $4096 \times 409$  & 409  & 0.565 & 0.572 & 123.256\\
         10 & 1*   & $4096 \times 409$  & 409  & 0.559 & 0.566 & 123.256\\
         1  & 100 & $4096 \times 40$   & 40   & 0.032 & 0.024 & 121.744\\
    \bottomrule
    \end{tabular}
    \caption{Results for downsampling the second fully-connected layer (FC-$4096\times4096$) of the VGG-16 (D) model for $224\times224$ images. For the attenuation factor (Attn. Factor) marked '*', we applied a downsampling factor of 10\% with a flat attenuation of one instead of 10. Hence, we observe a slightly degraded performance.}
    \label{tab:vgg16down}
\end{table}

Hence, we observe the promising applicability of our method in terms of up and downscaling deep neural networks. In the next section, we provide further discussion on the application of our methodology.

\section{Discussion}
\label{sec:discussion}
We observe the promising two-pronged applicability of our method. Firstly, in case of generating BTR images with Stable Diffusion, we attained competitive results, relying solely on modifying the convolution layer. Again, in terms of visible picture quality, we observe near-negligible differences along with solving the object-duplication problem. Secondly, we also demonstrated the general applicability of our method in interpolating layers beyond convolution, such as the fully-connected layers, when we interpolated whole deep neural networks to adapt for BTR inputs. Yet it is important to discuss further applicability of our methods in terms of reducing memory footprints for training deep neural networks.

\subsection{Reducing memory footprints with train-interpolate-fine-tune method}
One of the most crucial issues faced these days is the unavailability and high cost of DRAMs. A high-memory footprint of training data is significantly responsible for such inconvenience. Hence, utilizing our method, we believe the memory footprints could be drastically reduced. This is possible by shrinking or subsampling (notably audio, image, and video) data to lower dimensions, and training a neural network until convergence. It would be recommended to use the smallest $3\times3$ kernels of convolution for the initial training. Subsequently, the neural network can be interpolated for fine-tuning with the higher-dimensional data until satisfactory performance is achieved again. However, the fine-tuning step may not need to run for more than a few epochs as our experiments in \Cref{sec:dnn-interp} suggest the loss of performance after interpolation is not significant.

Since the majority of the training of the neural network is performed on low-dimensional data and neural networks, memory requirements are slashed in multiplicative order. For instance, if the initial training uses $\sim 4\times$ smaller input and neural network, the memory requirements would also be reduced at most $\sim 4\times$. This process is already shown in \Cref{sec:dnn-interp}, but we did not fine-tune to improve performance. However, in most applications of training neural-network models, the memory savings may be far less than $\sim 4\times$ due to the use of a smaller batch size or interpolating only a few layers of the model like ViT-B-16. Hence, we used a subset of the ImageNet-1k training set (4358 images) to demonstrate a practical scenario for the same models as before in \Cref{tab:memory-savings}. 

In this experiment, we used varying batch sizes for models and trained them for seven epochs to calculate the mean occupied GPU memory for both $1\times$ and $4\times$ training. But to calculate the mean and standard deviation of the used memory across epochs, we discarded the data from the first two epochs for CUDA warm-up \footnote{NVIDIA documentation on CUDA warm-up on Section 4.7.5.3: \url{https://docs.nvidia.com/cuda/cuda-programming-guide/04-special-topics/lazy-loading.html}}. \Cref{tab:memory-savings} shows that due to using smaller batch sizes, the contribution of memory savings is $4\times$ for training data but still minute compared to total used memory. Still, we get $~3\times$ and $~2\times$ memory savings for ResNet-18 and VGG-16 (AAP), respectively, as most of their layers were super-sampled. However, in the case of ViT-B-16, the memory savings was the least, as we only interpolated the convolution-projection layer and used a smaller batch size. 

However, this goes without saying that if a larger batch size is used, their contribution in used memory will increase linearly, eventually providing $\sim4\times$ memory-savings.

\begin{table}
\centering
\caption{The results for the memory-savings experiment. Here, Mean Memory refers to the amount of used GPU memory averaged across five out of seven epochs (provided with standard deviation). The first two epochs were not considered for CUDA warm-up. Data Memory refers to the amount of GPU memory used by a batch of data for training the models.}
\label{tab:memory-savings}
\resizebox{\textwidth}{!}{%
    \begin{tabular}{|c|c|c|c|c|c|c|}\hline
         \multirow{2}{*}{\textbf{Models}}&  \multirow{2}{*}{\textbf{Batch Size}} &  \multicolumn{2}{|c|}{\textbf{Data Memory (GB)}} & \multicolumn{2}{|c|}{\textbf{Mean Memory (GB)}} &  \multirow{2}{*}{\textbf{Memory-Savings}}\\
         \cline{3-6}        
        & & $1\times$&$4\times$ & $1\times$&$4\times$&\\\hline
         
         ResNet-18&  128& $0.072$&  $0.287$& $ 4.596\pm0$&  $13.860 \pm 0$ &$3.016 \times$\\\hline
         VGGNet-16 (AAP)&  32& $0.018$&  $0.072$ & $7.164 \pm 0$&   $14.149 \pm 0$ &$1.975\times$\\\hline
         ViT-B-16&  32&  $0.018$&  $0.072$ & $5.602\pm 0$&  $5.688 \pm 0$ &$1.015\times$\\ 
         \hline
    \end{tabular}
}    
    
\end{table}

\section{Conclusion}
\label{sec:conclusion}
In this work, we showed a theoretical method of interpolating convolutional layers to generate BTR images for Stable Diffusion without training. Simultaneously, we provided empirical evidence that backed our theoretical analysis. Our method achieved competitive performance against the state of the art with no modifications to the self-attention layers or low-pass filtering, unlike the methods compared with. Finally, we demonstrated the general applicability of our method on interpolating neural networks for accepting BTR inputs, where we interpolated all convolution layers and the first fully-connected layer to ensure adaptability. 

This may be promising, as models can be shrunk or expanded based on the input size with some loss in performance, and then fine-tuned. This may also save countless hours on adapting a model for a specific hardware, especially embedded systems, while reducing memory requirements for training.

\subsection{Limitations}
In terms of BTR image generation with SD-1.5, our method demonstrated competitive performance, but did not exceed the state of the art. Again, in the case of interpolating deep neural networks, we observed that not applying flat-attenuation to the last convolution layer of VGG-16 allowed for a better result (see \Cref{tab:minor-ablation-fa}). However, finding the optimum layer combination for VGG-16 or other models is beyond the scope of this study.

\begin{table}[h]
    \centering
    \begin{tabular}{ccc}\toprule
         Conv. layers w/o FA&  Accuracy&  F1-Score\\\midrule
         2 - 13&  0.532&  0.539\\
 12, 13& 0.658& 0.655\\
 13& 0.664& 0.66\\ \bottomrule
    \end{tabular}
    \caption{VGG-16 results on $448\times448$  ImageNet validation set without flat-attenuation on some of the convolution layers}
    \label{tab:minor-ablation-fa}
\end{table}

\subsection{Recommendations for Future Work}
Future work could focus on shrinking larger models, while using greater than $3\times3$ convolution on the base models for speed and efficiency. Additionally, quantizing downsampled models is worth investigating for potentially higher memory efficiency. Finally, training-free scaling of deep learning super-sampling (DLSS) models could be an interesting application domain to explore \citep{dlss}.

\bibliography{references}

@article{zhu2022efficientbicubic,
  title={An efficient bicubic interpolation implementation for real-time image processing using hybrid computing},
  author={Zhu, Yubin and Dai, Yonghang and Han, Kaining and Wang, Junchao and Hu, Jianhao},
  journal={Journal of Real-Time Image Processing},
  volume={19},
  number={6},
  pages={1211--1223},
  year={2022},
  publisher={Springer}
}

@inproceedings{smootheddilationwang,
author = {Wang, Zhengyang and Ji, Shuiwang},
title = {Smoothed Dilated Convolutions for Improved Dense Prediction},
year = {2018},
isbn = {9781450355520},
publisher = {Association for Computing Machinery},
address = {New York, NY, USA},
url = {https://doi.org/10.1145/3219819.3219944},
doi = {10.1145/3219819.3219944},
booktitle = {Proceedings of the 24th ACM SIGKDD International Conference on Knowledge Discovery \& Data Mining},
pages = {2486–2495},
numpages = {10},
location = {London, United Kingdom},
series = {KDD '18}
}

@inproceedings{fouriscale2024,
author = {Huang, Linjiang and Fang, Rongyao and Zhang, Aiping and Song, Guanglu and Liu, Si and Liu, Yu and Li, Hongsheng},
title = {FouriScale: A Frequency Perspective on Training-Free High-Resolution Image Synthesis},
year = {2024},
isbn = {978-3-031-73253-9},
publisher = {Springer-Verlag},
address = {Berlin, Heidelberg},
url = {https://doi.org/10.1007/978-3-031-73254-6_12},
doi = {10.1007/978-3-031-73254-6_12},
booktitle = {Computer Vision – ECCV 2024: 18th European Conference, Milan, Italy, September 29–October 4, 2024, Proceedings, Part XII},
pages = {196–212},
numpages = {17},
location = {Milan, Italy}
}

@misc{ldm-main,
      title={High-Resolution Image Synthesis with Latent Diffusion Models}, 
      author={Robin Rombach and Andreas Blattmann and Dominik Lorenz and Patrick Esser and Björn Ommer},
      year={2022},
      eprint={2112.10752},
      archivePrefix={arXiv},
      primaryClass={cs.CV},
      url={https://arxiv.org/abs/2112.10752}, 
}

@inproceedings{lin2024accdiffusion,
  title={AccDiffusion : An Accurate Method for Higher-Resolution Image Generation},
  author={Lin, Zhihang and Lin, Mingbao and Meng, Zhao and Ji, Rongrong},
  booktitle={ECCV},
  year={2024}
}

@inproceedings{du2024demofusion,
  title={DemoFusion: Democratising High-Resolution Image Generation With No \$\$\$},
  author={Du, Ruoyi and Chang, Dongliang and Hospedales, Timothy and Song, Yi-Zhe and Ma, Zhanyu},
  booktitle={CVPR},
  year={2024}
}

@article{bar2023multidiffusion,
  title={MultiDiffusion: Fusing Diffusion Paths for Controlled Image Generation},
  author={Bar-Tal, Omer and Yariv, Lior and Lipman, Yaron and Dekel, Tali},
  journal={arXiv preprint arXiv:2302.08113},
  year={2023}
}

@inproceedings{zhang2025hidiffusion,
  title={Hidiffusion: Unlocking higher-resolution creativity and efficiency in pretrained diffusion models},
  author={Zhang, Shen and Chen, Zhaowei and Zhao, Zhenyu and Chen, Yuhao and Tang, Yao and Liang, Jiajun},
  booktitle={European Conference on Computer Vision},
  pages={145--161},
  year={2024},
  organization={Springer}
}

@inproceedings{he2023scalecrafter,
  title={Scalecrafter: Tuning-free higher-resolution visual generation with diffusion models},
  author={He, Yingqing and Yang, Shaoshu and Chen, Haoxin and Cun, Xiaodong and Xia, Menghan and Zhang, Yong and Wang, Xintao and He, Ran and Chen, Qifeng and Shan, Ying},
  booktitle={The Twelfth International Conference on Learning Representations},
  year={2024}
}

@misc{nanobanana,
  title        = {Introducing Gemini 2.5 Flash Image, our state-of-the-art image model},
  author       = {Fortin, Alisa and Vernade, Guillaume and Kampf, Kat and Reshi, Ammaar},
  year         = {2025},
  month        = aug,
  day          = {26},
  howpublished = {https://developers.googleblog.com/introducing-gemini-2-5-flash-image/},
  note         = {Google Developers Blog},
}

@misc{labs2025flux1kontextflowmatching,
      title={FLUX.1 Kontext: Flow Matching for In-Context Image Generation and Editing in Latent Space}, 
      author={Black Forest Labs and Stephen Batifol and Andreas Blattmann and Frederic Boesel and Saksham Consul and Cyril Diagne and Tim Dockhorn and Jack English and Zion English and Patrick Esser and Sumith Kulal and Kyle Lacey and Yam Levi and Cheng Li and Dominik Lorenz and Jonas Müller and Dustin Podell and Robin Rombach and Harry Saini and Axel Sauer and Luke Smith},
      year={2025},
      eprint={2506.15742},
      archivePrefix={arXiv},
      primaryClass={cs.GR},
      url={https://arxiv.org/abs/2506.15742}, 
}

@misc{wu2025qwenimagetechnicalreport,
      title={Qwen-Image Technical Report}, 
      author={Chenfei Wu and Jiahao Li and Jingren Zhou and Junyang Lin and Kaiyuan Gao and Kun Yan and Sheng-ming Yin and Shuai Bai and Xiao Xu and Yilei Chen and Yuxiang Chen and Zecheng Tang and Zekai Zhang and Zhengyi Wang and An Yang and Bowen Yu and Chen Cheng and Dayiheng Liu and Deqing Li and Hang Zhang and Hao Meng and Hu Wei and Jingyuan Ni and Kai Chen and Kuan Cao and Liang Peng and Lin Qu and Minggang Wu and Peng Wang and Shuting Yu and Tingkun Wen and Wensen Feng and Xiaoxiao Xu and Yi Wang and Yichang Zhang and Yongqiang Zhu and Yujia Wu and Yuxuan Cai and Zenan Liu},
      year={2025},
      eprint={2508.02324},
      archivePrefix={arXiv},
      primaryClass={cs.CV},
      url={https://arxiv.org/abs/2508.02324}, 
}

@misc{openai2025gpt4oimage,
  title        = {Introducing 4o Image Generation},
  author       = {{OpenAI}},
  year         = {2025},
  month        = mar,
  day          = {25},
  howpublished = {\url{https://openai.com/index/introducing-4o-image-generation/}},
  note         = {OpenAI Product Blog},
}

@misc{bbc2026articleram,
  author       = {Gerken,Tom},
  title        = {Why everything from your phone to your PC may get pricier in 2026},
  year         = {2026},
  url          = {https://www.bbc.com/news/articles/c1dzdndzlxqo},
  note         = {Accessed: 2026-02-08}
}

@misc{micronexitram,
  author       = {Plungy, Mark and Kumar, Satya},
  title        = {Micron Announces Exit from Crucial Consumer Business},
  day          = {3},
  month        = dec,
  year         = {2025},
  url          = {https://investors.micron.com/news-releases/news-release-details/micron-announces-exit-crucial-consumer-business},
  note         = {Accessed: 2026-02-08}
}

@misc{gpuramprice,
  author       = {Martindale, Jon},
  title        = {Rumor Tips 15\% Price Hikes on GPUs From Asus, Gigabyte},
  day          = {20},
  month        = jan,
  year         = {2026},
  url          = {https://www.pcmag.com/news/rumor-tips-15-price-hikes-on-gpus-from-asus-gigabyte?test_uuid=04IpBmWGZleS0I0J3epvMrC&test_variant=B},
  note         = {Accessed: 2026-02-08}
}

@article{fid,
  title={Gans trained by a two time-scale update rule converge to a local nash equilibrium},
  author={Heusel, Martin and Ramsauer, Hubert and Unterthiner, Thomas and Nessler, Bernhard and Hochreiter, Sepp},
  journal={Advances in neural information processing systems},
  volume={30},
  year={2017}
}

@inproceedings{kid,
title={Demystifying {MMD} {GAN}s},
author={Mikołaj Bińkowski and Dougal J. Sutherland and Michael Arbel and Arthur Gretton},
booktitle={International Conference on Learning Representations},
year={2018},
url={https://openreview.net/forum?id=r1lUOzWCW},
}

@misc{nvidiaa100,
    author = {{NVIDIA}},
    title = {{NVIDIA A100 Tensor Core GPU}},
    howpublished = {\url{https://www.nvidia.com/en-us/data-center/a100/}},
    note = {Accessed: 2026-02-20}
}

@misc{intelxeon,
  author={{Intel}},
  title = {{Intel® Xeon® 6 Processors}},
  howpublished = {\url{https://www.intel.com/content/www/us/en/products/details/processors/xeon.html}},
  note = {Accessed: 2026-02-20}
}

@misc{simonyan2015deepconvolutionalnetworkslargescalevggnet,
      title={Very Deep Convolutional Networks for Large-Scale Image Recognition}, 
      author={Karen Simonyan and Andrew Zisserman},
      year={2015},
      eprint={1409.1556},
      archivePrefix={arXiv},
      primaryClass={cs.CV},
      url={https://arxiv.org/abs/1409.1556}, 
}

@inproceedings{he2016deepresnet,
  title={Deep residual learning for image recognition},
  author={He, Kaiming and Zhang, Xiangyu and Ren, Shaoqing and Sun, Jian},
  booktitle={Proceedings of the IEEE conference on computer vision and pattern recognition},
  pages={770--778},
  year={2016}
}

@inproceedings{vit,
title={An Image is Worth 16x16 Words: Transformers for Image Recognition at Scale},
author={Alexey Dosovitskiy and Lucas Beyer and Alexander Kolesnikov and Dirk Weissenborn and Xiaohua Zhai and Thomas Unterthiner and Mostafa Dehghani and Matthias Minderer and Georg Heigold and Sylvain Gelly and Jakob Uszkoreit and Neil Houlsby},
booktitle={International Conference on Learning Representations},
year={2021},
url={https://openreview.net/forum?id=YicbFdNTTy}
}

@article{imagenet,
Author = {Olga Russakovsky and Jia Deng and Hao Su and Jonathan Krause and Sanjeev Satheesh and Sean Ma and Zhiheng Huang and Andrej Karpathy and Aditya Khosla and Michael Bernstein and Alexander C. Berg and Li Fei-Fei},
Title = {{ImageNet Large Scale Visual Recognition Challenge}},
Year = {2015},
journal   = {International Journal of Computer Vision (IJCV)},
doi = {10.1007/s11263-015-0816-y},
volume={115},
number={3},
pages={211-252}
}

@misc{dlss,
      title={Deep Learning Techniques for Super-Resolution in Video Games}, 
      author={Alexander Watson},
      year={2020},
      eprint={2012.09810},
      archivePrefix={arXiv},
      primaryClass={cs.NE},
      url={https://arxiv.org/abs/2012.09810}, 
}
\bibliographystyle{unsrt}

 \appendix
\section{Theorems and Proofs}
\label{sec:theorems-and-proofs}
In this section, we provide our theorems, corollary, and their respective proof.

\ratiotheorem*

\begin{proof}
\label{pf:dot}
Let, $\mathbf{D},\mathbf{E}$ be a pair of $N$-dimensional vectors. We can express each of $\mathbf{D},\mathbf{E}$ vectors as a sum of discrete cosine waves with $\lfloor\frac{N}{2}\rfloor+1$ unique frequencies $k\in [0, \lfloor\frac{N}{2}\rfloor]$: 

\[
  \mathbf{D} = \left(\sum_{k=0}^{\lfloor\frac{N}{2}\rfloor} A_{k} \cos\left(2\pi\frac{k}{N}n + \phi_{k}\right)\right)_{n=0}^{N-1}  
\]

\[
  \mathbf{E} =  \left(\sum_{k=0}^{\lfloor\frac{N}{2}\rfloor} B_{k} \cos\left(2\pi\frac{k}{N}n + \theta_{k}\right)\right)_{n=0}^{N-1}
\]

Let $\operatorname{I}_s(\mathbf{V}, f)$ be an ideal supersampling function that supersamples a vector $\mathbf{V}$ given a factor $f$. Thus we supersample $\mathbf{D},\mathbf{E}$ by a factor $\alpha$ where $\alpha > 1$.

\[
    \mathbf{F} = \operatorname{I}_s(\mathbf{D}, \alpha) = \left(\sum_{k=0}^{\lfloor\frac{N}{2}\rfloor} A_{k} \cos\left(2\pi\frac{k}{\alpha N}n + \phi_{k}\right)\right)_{n=0}^{\alpha N-1}  
\]

\[
    \mathbf{G} = \operatorname{I}_s(\mathbf{E}, \alpha) =\left(\sum_{k=0}^{\lfloor\frac{N}{2}\rfloor} B_{k} \cos\left(2\pi\frac{k}{\alpha N}n + \theta_{k}\right)\right)_{n=0}^{\alpha N-1}  
\]

For $\mathbf{D}, \mathbf{E}, \mathbf{F}, \mathbf{G}$ we assume $\phi_0,\theta_0 = 0$ for all $N$ and $\phi_{\lfloor\frac{N}{2}\rfloor},\theta_{\lfloor\frac{N}{2}\rfloor} \in \left\{ 0, \pi \right\}$ when $N$ is even. The dot product between $\mathbf{D},\mathbf{E}$ can be shown using Parseval's theorem.

\[
   \langle \mathbf{D},\mathbf{E} \rangle= \frac{1}{N} \left( \sum_{k=0}^{N-1} \widehat{D}[k] \cdot \overline{\widehat{E}[k]} \right)
\]

When $N$ is even:

\[
   \langle \mathbf{D},\mathbf{E} \rangle= \frac{1}{N} \left(N^2 A_0 B_0 \exp\left(0\right) + \sum_{k=1}^{\frac{N}{2}-1}  \widehat{D}[k] \cdot \overline{\widehat{E}[k]}  + N^2 A_{\frac{N}{2}} B_{\frac{N}{2}} \exp\left(i\left(\phi_{\frac{N}{2}}-\theta_{\frac{N}{2}}\right)\right) + \sum_{k=\frac{N}{2}+1}^{N-1} \widehat{D}[k] \cdot \overline{\widehat{E}[k]}\right)
\]
\[
   \langle \mathbf{D},\mathbf{E} \rangle= \frac{1}{N} \left(N^2 A_0 B_0 + \sum_{k=1}^{\frac{N}{2}-1} \widehat{D}[k] \cdot \overline{\widehat{E}[k]} + \sum_{k=-(\frac{N}{2}-1)}^{-1} \widehat{D}[k] \cdot \overline{\widehat{E}[k]}  + N^2 A_{\frac{N}{2}} B_{\frac{N}{2}} \cos\left(\phi_{\frac{N}{2}}-\theta_{\frac{N}{2}}\right)\right)
\]
\[
   \langle \mathbf{D},\mathbf{E} \rangle= \frac{1}{N} \left(N^2 A_0 B_0 + \sum_{k=1}^{\frac{N}{2}-1}\left(\widehat{D}[k] \cdot \overline{\widehat{E}[k]} + \widehat{D}[-k] \cdot \overline{\widehat{E}[-k]} \right) + N^2 A_{\frac{N}{2}} B_{\frac{N}{2}} \cos\left(\phi_{\frac{N}{2}}-\theta_{\frac{N}{2}}\right)\right)
\]
\[
   \langle \mathbf{D},\mathbf{E} \rangle= \frac{1}{N} \left(N^2 A_0 B_0 + \sum_{k=1}^{\frac{N}{2}-1} \frac{N^2 A_k B_k}{4} \left( \exp\left(i\left(\phi_k-\theta_k\right)\right) +\exp\left(-i\left(\phi_k-\theta_k\right)\right)\right)+ N^2 A_{\frac{N}{2}} B_{\frac{N}{2}} \cos\left(\phi_{\frac{N}{2}}-\theta_{\frac{N}{2}}\right) \right)
\]

\[
   \langle \mathbf{D},\mathbf{E} \rangle= \frac{1}{N} \left(N^2 A_0 B_0 + \sum_{k=1}^{\frac{N}{2}-1} \frac{N^2 A_k B_k}{2} \cos\left(\phi_k-\theta_k\right) + N^2 A_{\frac{N}{2}} B_{\frac{N}{2}} \cos\left(\phi_{\frac{N}{2}}-\theta_{\frac{N}{2}}\right) \right)
\]

\[
   \langle \mathbf{D},\mathbf{E} \rangle= N \left( A_0 B_0 + \sum_{k=1}^{\frac{N}{2}-1} \frac{A_k B_k}{2} \cos\left(\phi_k-\theta_k\right) + A_{\frac{N}{2}} B_{\frac{N}{2}} \cos\left(\phi_{\frac{N}{2}}-\theta_{\frac{N}{2}}\right) \right)
\]

When $N$ is odd:

\[
   \langle \mathbf{D},\mathbf{E} \rangle= \frac{1}{N} \left(N^2 A_0 B_0 \exp\left(0\right) + \sum_{k=1}^{\frac{N-1}{2}} \widehat{D}[k] \cdot \overline{\widehat{E}[k]} + \sum_{k=\frac{N+1}{2}}^{N-1} \widehat{D}[k] \cdot \overline{\widehat{E}[k]}\right)
\]
\[
   \langle \mathbf{D},\mathbf{E} \rangle= \frac{1}{N} \left(N^2 A_0 B_0 + \sum_{k=1}^{\frac{N-1}{2}} \widehat{D}[k] \cdot \overline{\widehat{E}[k]} + \sum_{k=-\frac{N-1}{2}}^{-1} \widehat{D}[k] \cdot \overline{\widehat{E}[k]} \right)
\]
\[
   \langle \mathbf{D},\mathbf{E} \rangle= \frac{1}{N} \left(N^2 A_0 B_0 + \sum_{k=1}^{\frac{N-1}{2}} \left( \widehat{D}[k] \cdot \overline{\widehat{E}[k]} + \widehat{D}[-k] \cdot \overline{\widehat{E}[-k]} \right) \right)
\]
\[
   \langle \mathbf{D},\mathbf{E} \rangle= N \left( A_0 B_0 + \sum_{k=1}^{\frac{N-1}{2}} \frac{A_k B_k}{2} \cos\left(\phi_k-\theta_k\right)\right)
\]

We combine both cases with $\delta_{even} (L)$.
\[ 
    \operatorname{\delta_{even}}(L) = 
    \begin{cases} 
      1 & \text{if $L$ is even }\\
      0 & \text{if $L$ is odd }
    \end{cases}
\]
Hence the final equation:
\begin{equation}
\label{eq:dot-prod-sk-th}
   \langle \mathbf{D},\mathbf{E} \rangle= N \left( A_0 B_0 + \sum_{k=1}^{\lfloor\frac{N-1}{2}\rfloor} \frac{A_k B_k}{2} \cos\left(\phi_k-\theta_k\right) + \operatorname{\delta_{even}}(N) A_{\frac{N}{2}} B_{\frac{N}{2}} \cos\left(\phi_{\frac{N}{2}}-\theta_{\frac{N}{2}}\right) \right)
\end{equation}

Similarly for $\mathbf{F},\mathbf{G}$ we can show:
\[
   \langle \mathbf{F},\mathbf{G} \rangle= \frac{1}{\alpha N} \left( \sum_{k=0}^{\alpha N-1} \widehat{F}[k] \cdot \overline{\widehat{G}[k]} \right)
\]

\[
    \langle \mathbf{F},\mathbf{G} \rangle= \alpha N \left( A_0 B_0 + \sum_{k=1}^{\lfloor\frac{\alpha N-1}{2}\rfloor} \frac{A_k B_k}{2} \cos\left(\phi_k-\theta_k\right) + \operatorname{\delta_{even}}(\alpha N) A_{\frac{\alpha N}{2}} B_{\frac{\alpha N}{2}} \cos\left(\phi_{\frac{\alpha N}{2}}-\theta_{\frac{\alpha N}{2}}\right) \right)
\]

However, as $\mathbf{F},\mathbf{G}$ are supersampled versions of $\mathbf{D},\mathbf{E}$ respectively, $\alpha \widehat{D}[k] = \widehat{F}[k]$ and $\alpha \widehat{E}[k] = \widehat{G}[k]$ for all $k \in [-\lfloor \frac{N-1}{2}\rfloor, \lfloor \frac{N-1}{2}\rfloor]$. For even $N$, we also have $\frac{\alpha}{2}\widehat{D}[\frac{N}{2}] = \widehat{F}[\pm \frac{N}{2}]$ and $\frac{\alpha}{2}\widehat{E}[\frac{N}{2}] = \widehat{G}[\pm \frac{N}{2}]$. For the rest of the frequencies, $|k| > \lfloor \frac{N}{2}\rfloor$, $\widehat{F}[k]=\widehat{G}[k]=0$.

\begin{equation}
\label{eq:dot-prod-an-sk-th}
     \langle \mathbf{F},\mathbf{G} \rangle= \alpha N \left( A_0 B_0 + \sum_{k=1}^{\lfloor\frac{N-1}{2}\rfloor} \frac{A_{k} B_{k}}{2} \cos\left(\phi_{k}-\theta_{k}\right) + \operatorname{\delta_{even}}(N) \frac{A_{\frac{N}{2}} B_{\frac{N}{2}}}{2} \cos\left(\phi_{\frac{N}{2}}-\theta_{\frac{N}{2}}\right) \right)  
\end{equation}

 For odd $N$ or $A_{\frac{N}{2}}B_{\frac{N}{2}}=0$, the Nyquist term vanishes, yielding the identity:

\begin{equation}
\label{eq:main-true-th}
     \frac{1}{\alpha} \langle \mathbf{F},\mathbf{G} \rangle= N \left( A_0 B_0 + \sum_{k=1}^{\lfloor\frac{N-1}{2}\rfloor} \frac{A_{k} B_{k}}{2} \cos\left(\phi_{k}-\theta_{k}\right) \right)  
\end{equation}

\begin{equation}
    \label{eq:equality-th}
     \frac{1}{\alpha} \langle \mathbf{F},\mathbf{G} \rangle = \langle \mathbf{D},\mathbf{E} \rangle
\end{equation}
\end{proof}

\section{Implementation in SD-1.5}
\label{sec:impl-sd-1.5}

\begin{figure}[!ht]
    \centering 
    \begin{subfigure}{0.4\textwidth}
        \centering        \includegraphics[width=\linewidth]{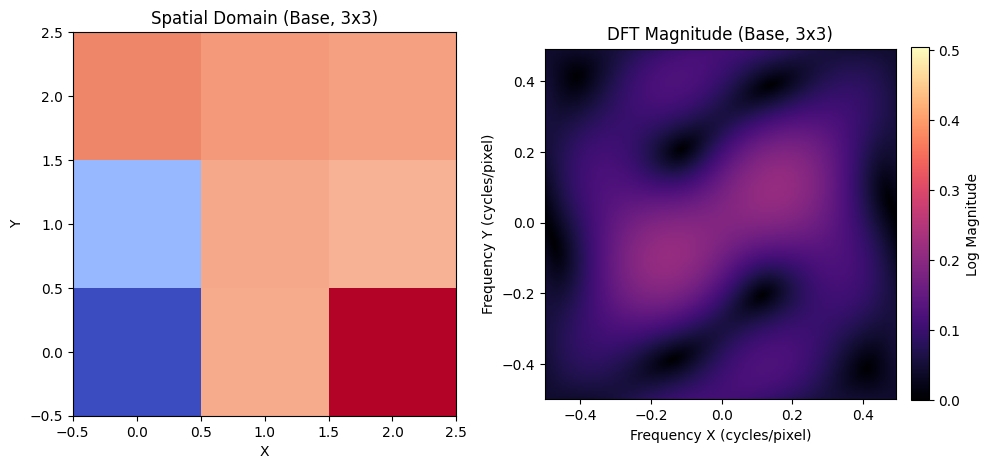}
        \caption{Base $3 \times 3$ kernel}
        \label{fig:base}
    \end{subfigure} \\
    \begin{subfigure}{0.4\textwidth}
        \centering      \includegraphics[width=\linewidth]{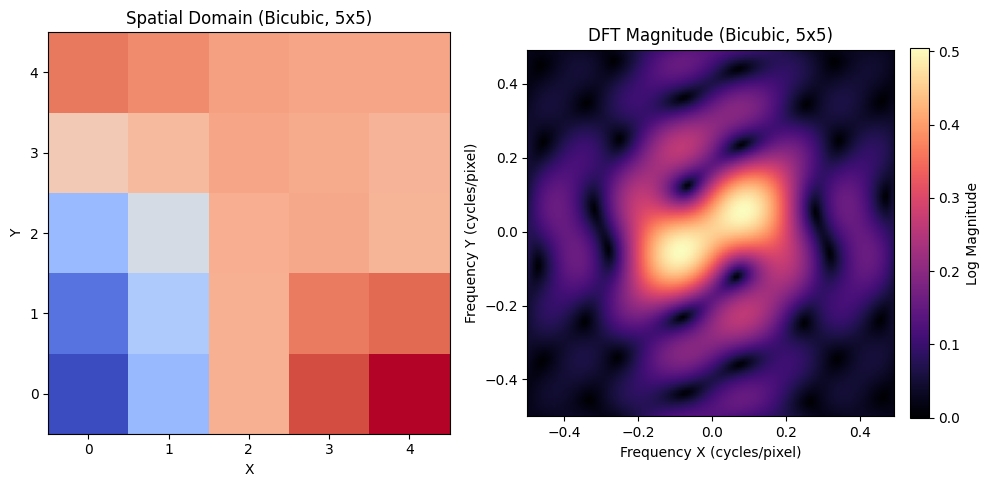}
        \caption{Interpolated $5 \times 5$ kernel (bicubic)}
        \label{fig:bicubic}
    \end{subfigure}
         \begin{subfigure}{0.4\textwidth}
        \centering    \includegraphics[width=\linewidth]{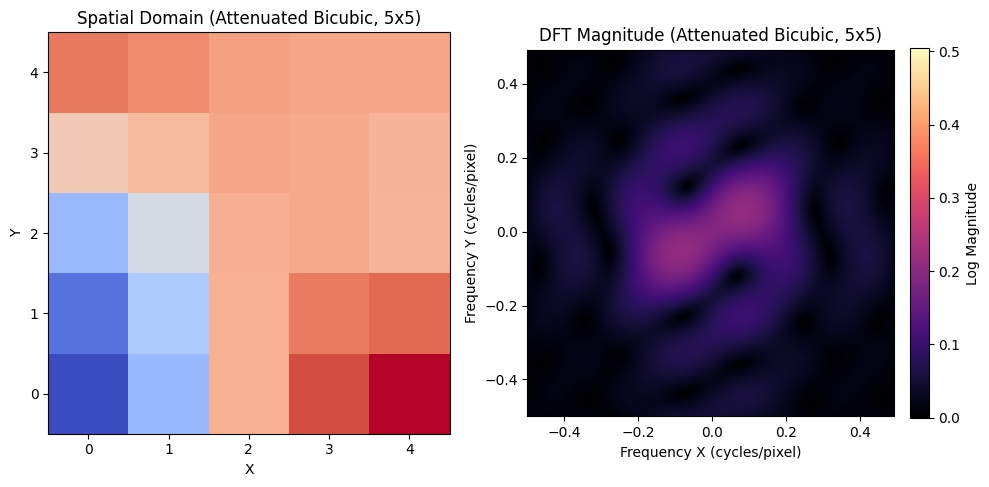}
        \caption{Attenuated $5 \times 5$ kernel (bicubic)}
        \label{fig:bicubic-at}
    \end{subfigure}
    \begin{subfigure}{0.4\textwidth}
        \centering
        \includegraphics[width=\linewidth]{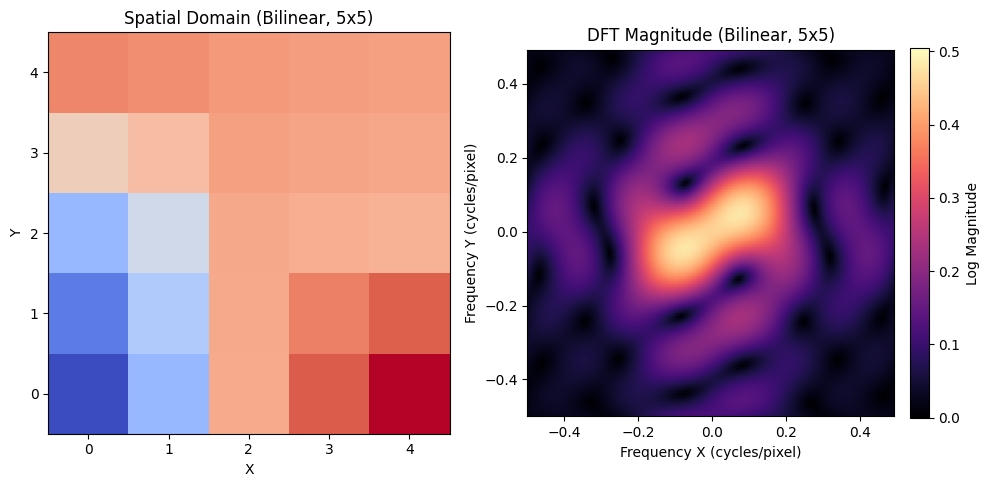}
        \caption{Interpolated $5 \times 5$ kernel (bilinear)}
        \label{fig:bilinear}
    \end{subfigure}
        \begin{subfigure}{0.4\textwidth}
        \centering    \includegraphics[width=\linewidth]{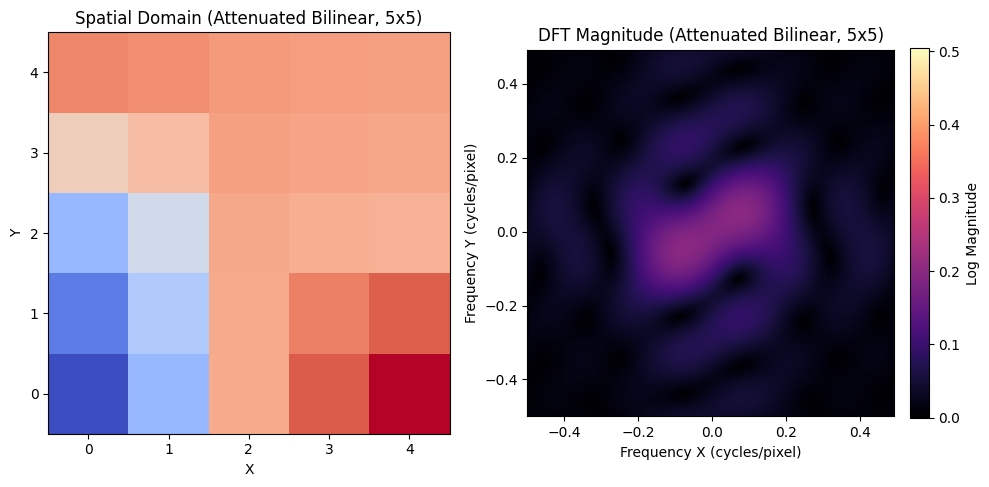}
        \caption{Attenuated $5 \times 5$ kernel (bilinear)}
        \label{fig:bilinear-at}
    \end{subfigure}
    \begin{subfigure}{0.4\textwidth}
        \centering
        \includegraphics[width=\linewidth]{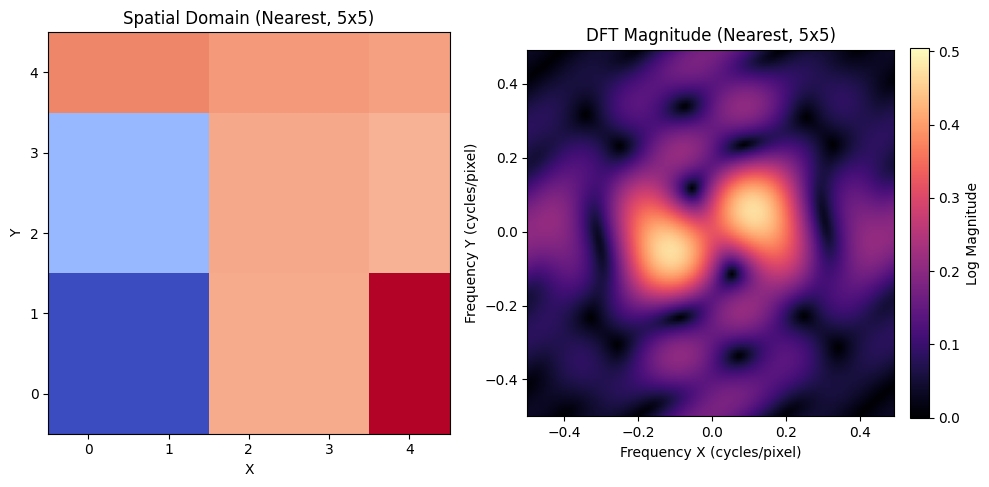}
        \caption{Interpolated $5 \times 5$ kernel (nearest)}
        \label{fig:nearest}
    \end{subfigure}
    \begin{subfigure}{0.4\textwidth}
        \centering
    \includegraphics[width=\linewidth]{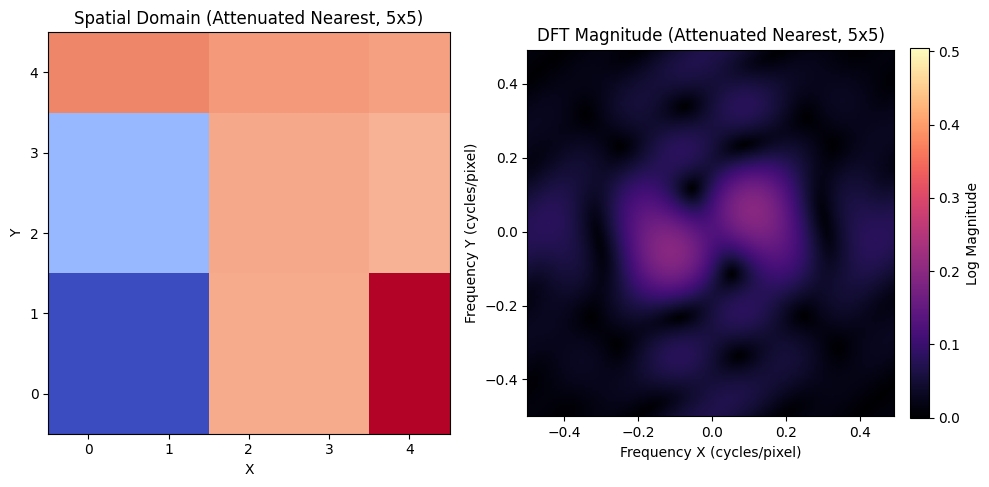}
        \caption{Attenuated $5 \times 5$ kernel (nearest)}
        \label{fig:nearest-at}
    \end{subfigure}
    \begin{subfigure}{0.4\textwidth}
        \centering
    \includegraphics[width=\linewidth]{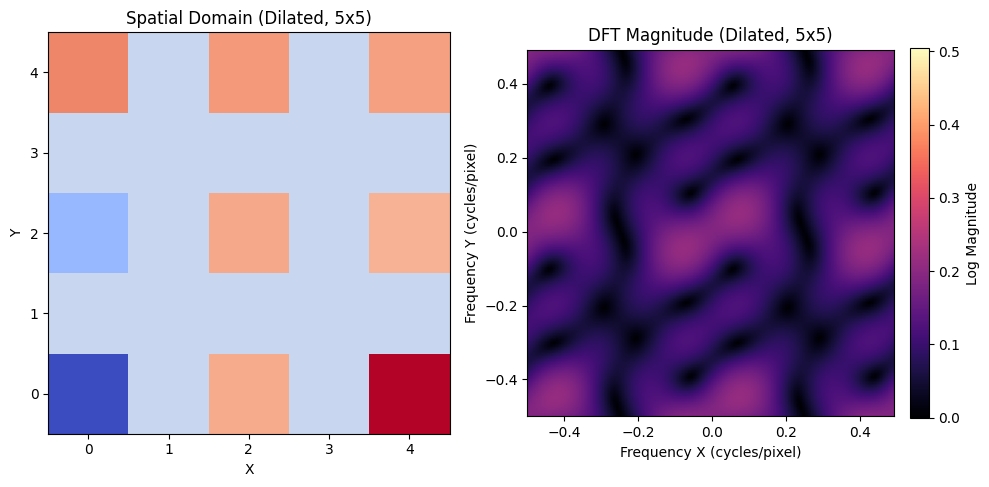}
        \caption{Interpolated $5 \times 5$ kernel (dilation)}
        \label{fig:dilated}        
    \end{subfigure}
    \caption{Comparison of kernel interpolation methods. The plots on the left and right demonstrate spatial and spectral dimensions, respectively. The spectral dimensions are plotted as log-one-plus of the DFT spectrum.}
\label{fig:kernel-interplation-comparison} 
\end{figure}

\Cref{fig:kernel-interplation-comparison} demonstrates a curious phenomenon. We observe that both \Cref{fig:base} \footnote{The base kernel was plotted from SD-1.5 with module name, \texttt{down\_blocks.2.resnets.0.conv1}.} and \Cref{fig:dilated} have lesser intensities than the interpolated kernels. Again, the interpolated and attenuated kernels (in \Cref{fig:bicubic-at,fig:bilinear-at,fig:nearest-at}) shows a clear similarity in intensity with \Cref{fig:base,fig:dilated}. Moreover, several studies have already shown that dilated convolution can generate images beyond the training resolution \citep{fouriscale2024,zhang2025hidiffusion}. Furthermore, \Cref{th:dot-main} suggests the attenuation of kernel intensities for equivalent outputs in the case of interpolation. This provides the basis for our kernel-interpolation technique for SD models shown in \Cref{eq:main-approx}. Hence, we multiply the kernel amplitude by the inverse of the supersampling factor to get our equivalent supersampled kernels. This means if we intend to generate $aM\times bN$ resolution images from SD models, we need to scale the kernel by $\frac{1}{a\cdot b}$ provided the default (training) resolution is $M \times N$.

\begin{align}
    H_{s} &= a * (H-1) + 1 \\
    W_{s} &= b * (W-1) + 1  
\label{eq:kernel-interplation-formula}
\end{align}

However, since convolution kernels utilized in the SD pipeline are odd to begin with, we kept them odd. For instance, for $4\times$ supersampling, $3\times3$ kernels were supersampled to $5\times5$ instead of $6\times6$ following related works \citep{fouriscale2024}.  This is done following \Cref{eq:kernel-interplation-formula}, where $a$ \& $b$ stand for the supersampling factor for the height and width of the latent image, respectively. $H_s$ \& $W_s$ stand for the interpolated, whereas $H$ \& $W$ stand for the default height and width of the kernel, respectively. Some other modifications to the SD pipeline include ---

\begin{itemize}
    \item \textit{Convolution Sublayer Selection. }We observe better results when interpolating most of the kernels than all of them. \Cref{fig:unet-highlighted-kernels} highlights the layers whose convolution sub-layers are interpolated; the rest are kept as is. 
    \item \textit{Denoising Step Range.} Again, instead of denoising for all 50 DDIM timesteps, we denoise for the first 31 steps (from step 0 to 30).
\end{itemize}

% \begin{figure}
% \centering

%     \caption{Attenuated and interpolated version of \Cref{fig:bicubic,fig:bilinear,fig:nearest}. The interpolated kernel shows a clear similarity in appearance \textit{and} intensity with \Cref{fig:base,fig:dilated}.}
% \end{figure}

\begin{figure}[h]
    \centering
    \includegraphics[width=0.7\linewidth]{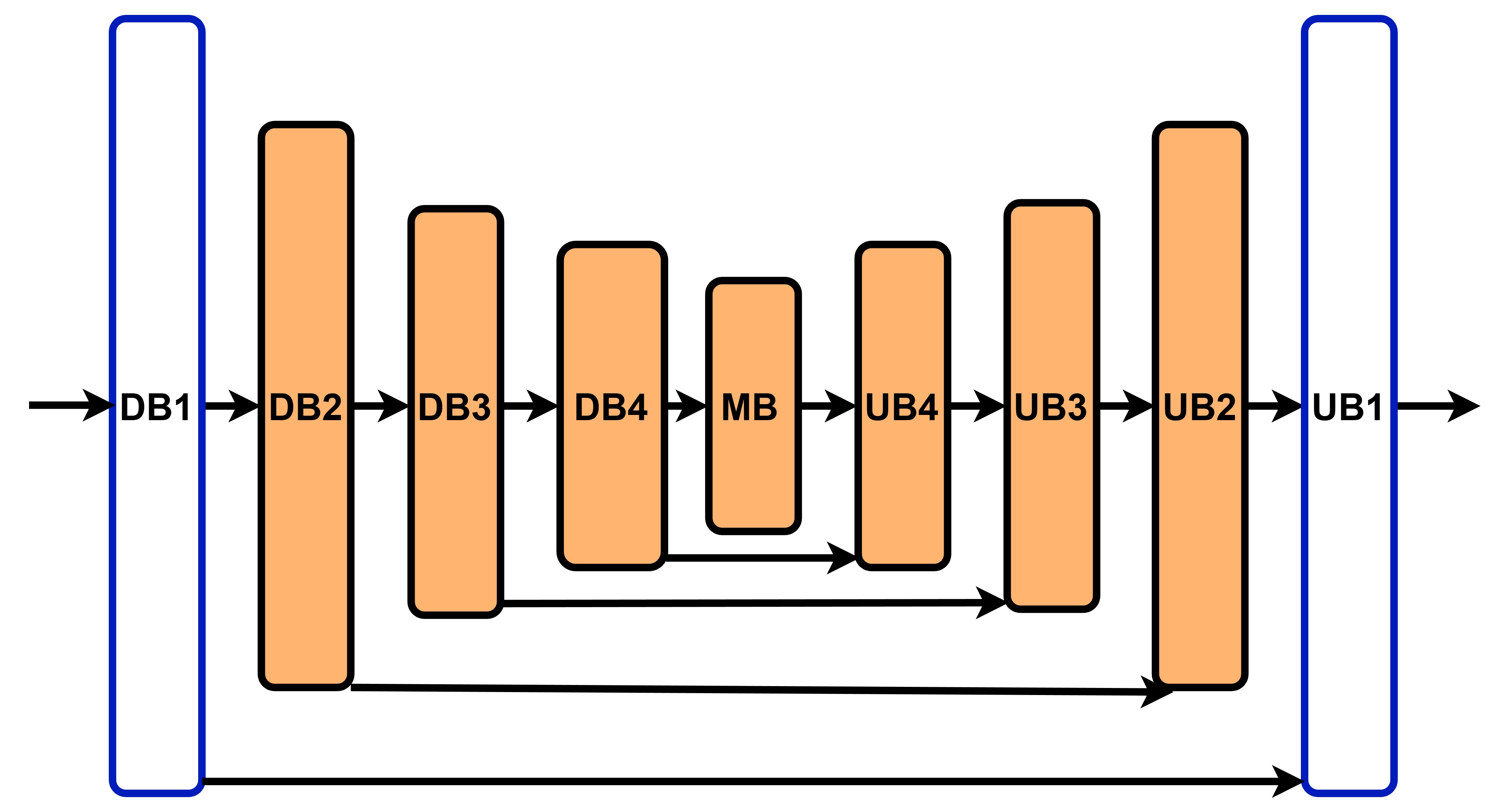}
    \caption{UNet model architecture for denoising latents of the Stable Diffusion pipeline. The convolution sublayers of the layers marked orange are interpolated for the best results. They are selected through empirical analysis, as suggested by \citep{fouriscale2024}.}
    \label{fig:unet-highlighted-kernels}
\end{figure}

The rest of the settings are kept as is.

\section{Why dilation may converge more slowly?}
\label{sec:interp-slow}
A valid inquiry could be why dilation would be less effective in convergence during finetuning. In order to test this, we fine-tune the ResNet-18 architecture on the ImageNet-1k train set already pretrained on ImageNet. For this, we used $448\times448$ images as in \Cref{sec:dnn-interp}. We used a stochastic gradient descent optimizer with a learning rate of $1 \times 10^{-3}$ and a momentum of $0.9$.

\begin{figure}[h]
    \centering
    \begin{subfigure}{0.45\textwidth}
        \centering
        \includegraphics[width=\linewidth]{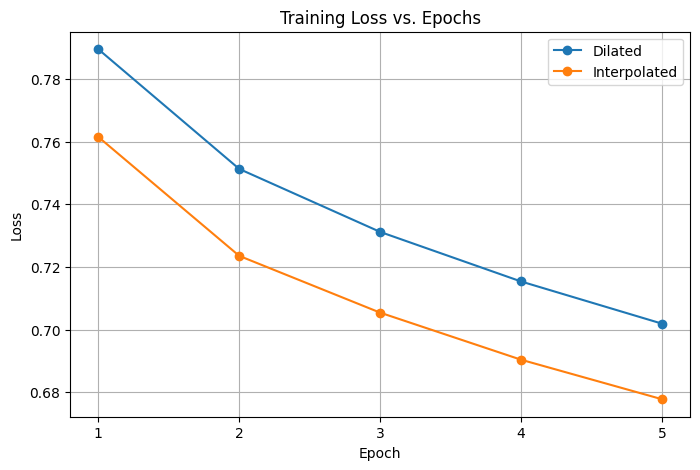}
        \caption{Training loss}
        \label{fig:training_loss}
    \end{subfigure}
        \begin{subfigure}{0.45\textwidth}
        \centering    \includegraphics[width=\linewidth]{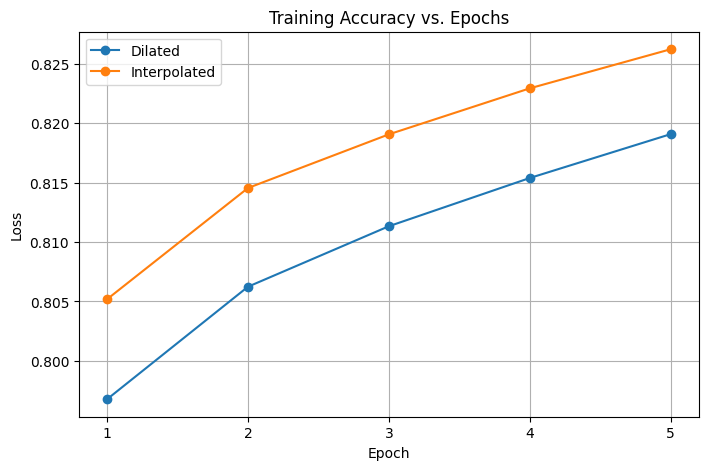}
        \caption{Training accuracy}
        \label{fig:training_acc}
    \end{subfigure}
    \caption{Training loss and accuracy comparison between dilated and interpolated ResNet-18, finetuned for five epochs.}
    \label{fig:dvifinetuning}
\end{figure}

We can observe in \Cref{fig:dvifinetuning} that interpolated-finetuning converges better and has lower loss. After finetuning, we observe the final result on the ImageNet validation set in\Cref{tab:table-finetuning-dilvint}. In the post-finetuning results, we observe a clear improvement of the interpolated ResNet-18 over the dilated version.

\begin{table}
    \centering
    \begin{tabular}{ccc}\toprule
         Interp. type&  Accuracy& F1-Score\\\midrule
 Dilation& 0.709&0.705\\
         Interpolation&  0.719& 0.716\\ \bottomrule
    \end{tabular}
    \caption{Post-finetuning results on $448\times448$ ImageNet validation set of ResNet-18.}
    \label{tab:table-finetuning-dilvint}
\end{table}

However, this may not apply to all neural networks in the same way. But we believe, in general, the interpolated kernels may converge faster than dilated kernels due to their proximity to ideal supersampled kernels. 

\section{VGG-16 architecture and modification details}
This section consists of VGG-16 (D) architectural changes relative to the base model for adapting to $4 \times$ input in \Cref{tab:vgg16bvi}.

\begin{table}[h]
\centering
\caption{Base, interpolated architecture with and without adaptive average pooling for VGG-16 (D). The padding was also changed for all convolution layers as discussed earlier, but omitted from this table for brevity.}
\begin{tabular}{|c|c|c|}
\hline
Base & Interp. & Interp. + AAP\\
\hline
Input-$224\times224$ & \multicolumn{2}{|c|}{Input-$448\times448$} \\ 
\hline
 conv3-64 & conv5-64 & conv5-64\\
 conv3-64 & conv5-64 & conv5-64\\
\hline
\multicolumn{3}{|c|}{maxpool} \\
\hline
conv3-128 & conv5-128 & conv5-128 \\
conv3-128 & conv5-128 & conv5-128\\
\hline
\multicolumn{3}{|c|}{maxpool} \\
\hline
 conv3-256 & conv5-256 & conv5-256 \\
 conv3-256 & conv5-256 & conv5-256 \\
 conv3-256 & conv5-256 & conv5-256 \\
\hline
\multicolumn{3}{|c|}{maxpool} \\
\hline
 conv3-512 & conv5-512 & conv5-512\\
 conv3-512 & conv5-512 & conv5-512\\
 conv3-512 & conv5-512 & conv5-512\\
\hline
\multicolumn{3}{|c|}{maxpool} \\
\hline
 conv3-512 & conv5-512 & conv5-512 \\
 conv3-512 & conv5-512 & conv5-512\\
 conv3-512 & conv5-512 & conv5-512\\
\hline
\multicolumn{3}{|c|}{maxpool} \\
\hline
\text{} & \text{} & Adaptive Avg. Pool (7,7)\\
\hline
\multicolumn{3}{|c|}{output shape} \\
\hline
 $512, 7, 7$ & $512, 14, 14$ & $512, 7, 7$\\
\hline
\multicolumn{3}{|c|}{flattened shape} \\
\hline
 $25088$ & $100352$ & $25088$\\
\hline
 FC-$4096\times25088$ & FC-$4096\times100352$ & FC-$4096\times25088$\\
\hline
\multicolumn{3}{|c|}{FC-$4096\times4096$*} \\
\multicolumn{3}{|c|}{FC-$4096\times1000$} \\
\multicolumn{3}{|c|}{softmax} \\
\hline
\end{tabular}
\label{tab:vgg16bvi}
\end{table}

\end{document}